\documentclass[10pt,twocolumn,letterpaper]{article}

\usepackage{iccv}
\usepackage{times}
\usepackage{epsfig}
\usepackage{graphicx}
\usepackage{amsmath}
\usepackage{amssymb}
\usepackage{amsthm}
\usepackage{booktabs}
\usepackage{bm}
\usepackage{mathtools}
\usepackage{enumitem}

\usepackage[pagebackref=true,breaklinks=true,letterpaper=true,colorlinks,bookmarks=false]{hyperref}

\usepackage[capitalize]{cleveref}
\crefname{section}{Sec.}{Secs.}
\Crefname{section}{Section}{Sections}
\Crefname{table}{Table}{Tables}
\crefname{table}{Tab.}{Tabs.}

\usepackage{subcaption}

\iccvfinalcopy 

\ificcvfinal\fi

\newcommand{\rvx}{\mathbf{x}}
\newcommand{\rvy}{\mathbf{y}}
\newcommand{\rvz}{\mathbf{z}}
\newcommand{\bbi}{\mathbf{1}}

\newcommand{\cx}{\mathcal{X}}
\newcommand{\cy}{\mathcal{Y}}
\newcommand{\cz}{\mathcal{Z}}
\newcommand{\cd}{\mathcal{D}}

\newcommand{\rmtrain}{\mathrm{train}}
\newcommand{\rmtest}{\mathrm{test}}

\theoremstyle{plain}
\newtheorem{definition}{Definition}
\newtheorem{theorem}{Theorem}

\newtheorem{lemma}{Lemma}

\theoremstyle{definition}

\begin{document}

\title{Preserving Silent Features for Domain Generalization}

\author{Chujie Zhao\qquad Tianren Zhang \qquad Feng Chen\\
Department of Automation, Tsinghua University\\
{\tt\small \{zhaocj22, zhangtr22\}@mails.tsinghua.edu.cn; chenfeng@mail.tsinghua.edu.cn}
}

\maketitle
\ificcvfinal\thispagestyle{empty}\fi

\begin{abstract}
   Domain generalization (DG) aims to improve the generalization ability of the model trained on several known training domains over unseen test domains. Previous work has shown that self-supervised contrastive pre-training improves the robustness of the model on downstream tasks. However, in this paper, we find that self-supervised models do not exhibit better generalization performance than supervised models pre-trained on the same dataset in the DG setting. We argue that this is owing to the fact that the richer intra-class discriminative features extracted by self-supervised contrastive learning, which we term silent features, are suppressed during supervised fine-tuning. These silent features are likely to contain features that are more generalizable on the test domain. In this work, we model and analyze this feature suppression phenomenon and theoretically prove that preserving silent features can achieve lower expected test domain risk under certain conditions. In light of this, we propose a simple yet effective method termed STEP (\underline{\textbf{S}}ilen\underline{\textbf{t}} F\underline{\textbf{e}}ature \underline{\textbf{P}}reservation) to improve the generalization performance of the self-supervised contrastive learning pre-trained model by alleviating the suppression of silent features during the supervised fine-tuning process. Experimental results show that STEP exhibits state-of-the-art performance on standard DG benchmarks with significant distribution shifts.
\end{abstract}

\section{Introduction}
\label{sec:intro}

Recent advances in computer vision tasks utilizing deep learning~\cite{ILSVRC15} under the independently and  identically distributed (i.i.d.) assumption~\cite{mohri2018foundations} are truly impressive.
However, simply deploying deep models in real-world situations frequently results in considerable performance degradation due to different distributions between training and test data, such as image renditions for image recognition and weather conditions for autonomous driving~\cite{Hendrycks_2021_ICCV, Wu_2021_ICCV}.

The goal of Domain Generalization (DG) is to enhance the out-of-distribution (OOD) generalization by mimicking distribution shifts by manually dividing data into several domains according to their characteristics, e.g. style, context, camera types, etc. 
In contrast to domain adaptation, DG assumes to train models on source domains that can generalize to completely inaccessible target domains, which is a more challenging and realistic setting~\cite{pmlr-v28-muandet13}.

\begin{figure}[t]
  \centering
   \begin{subfigure}{0.9\linewidth}
    \includegraphics[width=1.0\linewidth]{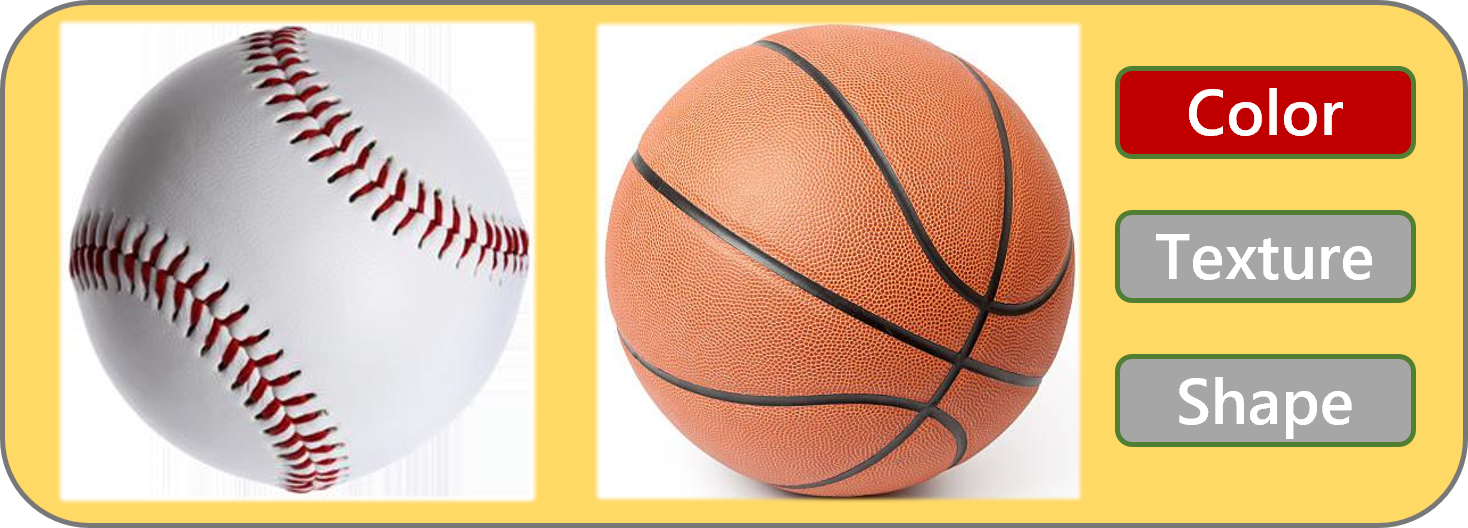}
    \caption{real domain}
    \label{fig:real}
  \end{subfigure}
  \hfill
  \begin{subfigure}{0.9\linewidth}
    \includegraphics[width=1.0\linewidth]{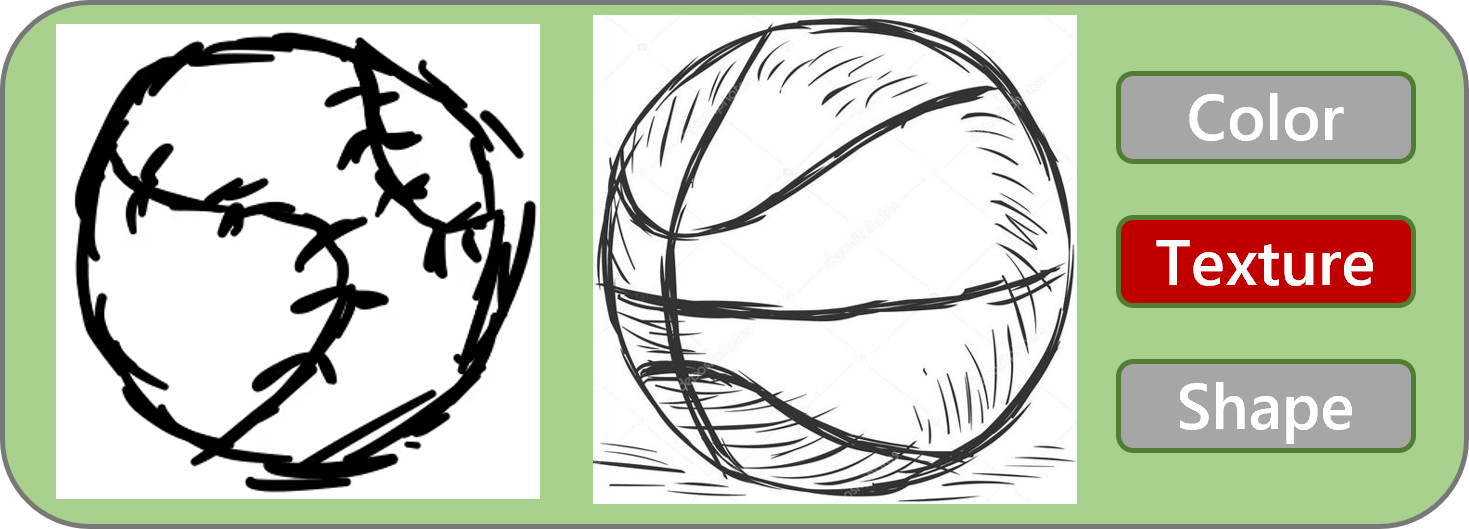}
    \caption{sketch domain}
    \label{fig:sketch}
  \end{subfigure}

   \caption{Examples of the features used to recognize baseball and basketball in DomainNet~\cite{Peng_2019_ICCV_domainnet}, dominant features are highlighted in red, whereas silent features are gray.}
   \label{fig:feature}
\end{figure}

Previous research in the DG literature generally employs a paradigm of utilizing the backbone supervised pre-trained on ImageNet~\cite{deng2009imagenet} and then fine-tuning the model on downstream datasets to achieve better generalization performance. For example, the common baseline, ERM~\cite{vapnik1999nature}, applies supervised fine-tuning on all source domains of the downstream DG dataset together and outperforms most of the DG methods on DomainBed~\cite{gulrajani2021in}. Existing studies have demonstrated that self-supervised pre-training can improve the robustness of models on downstream tasks~\cite{NEURIPS2019_selfrobust}, and some recent self-supervised contrastive learning methods based on instance discrimination~\cite{Wu_2018_CVPR,He_2020_CVPR,pmlr-v119-chen20j,NEURIPS2020_70feb62b} perform strongly on a wide range of tasks. Unfortunately, this is not the case in the DG setting~\cite{miro}. In this work, we also empirically find that the direct combination of self-supervised contrastive learning models pre-trained on ImageNet and ERM does not yield better performance, and sometimes even performs worse than supervised learning pre-trained models. We conjecture that this is caused by the downstream ERM method's supervised fine-tuning phase, which suppresses the intra-class features extracted by the self-supervised contrastive learning algorithm. This feature suppression phenomenon is due to supervised learning’s strong label bias forcing the model to pursue discriminative features between categories excessively, which neglects intra-class features that are weaker for classification tasks on available data distributions. Thus, we refer to this subset of the features that extracted by self-supervised contrastive learning as \textit{silent features}, and their counterparts as dominant features. As shown in~\cref{fig:feature}, we provide a visual example that the color is the dominant feature on the real domain, which is most discriminative for the classification of basketball and tennis, but completely useless on the sketch domain. The texture feature obtained by comparison within the basketball category is instead more discriminative on the sketch domain.

In this paper, we formulate the feature suppression process and theoretically demonstrate that lower test domain risk can be achieved by leveraging silent features under certain conditions. Therefore, we attempt to improve the DG performance of the model by preserving silent features obtained by self-supervised contrast learning pre-training as much as possible.
We propose a simple yet effective method called \underline{\textbf{S}}ilen\underline{\textbf{t}} F\underline{\textbf{e}}ature \underline{\textbf{P}}reservation (STEP) to alleviate the suppression of silent features caused by supervised fine-tuning and to improve the convergence stability of self-supervised contrastive pre-trained backbones.
STEP consists of two main components: 1) linear probing then full fine-tuning (LP-FT)~\cite{kumar2022finetuning} strategy to prevent the initial features from being suppressed during downstream fine-tuning; 2) Stochastic Weight Averaging Densely (SWAD)~\cite{NEURIPS2021_swad} technique to find flat minima by averaging the model's weights over time steps. Our approach reveals the strong OOD generalization ability of the self-supervised contrastive pre-trained model.

To rigorously evaluate our proposed method, we mainly followed the network architectures and model selection criteria of  DomainBed~\cite{gulrajani2021in}.
Experimental results on several widely used standard DG datasets validate our theory while showing that the performance of our simple approach is comparable to the current state-of-the-art methods. We also conduct comprehensive ablation studies to justify the role of each component in STEP and the impact of different self-supervised contrastive learning algorithms.

Our main contributions are as follows:
\begin{itemize}
\item We provide a novel perspective for enhancing the OOD generalization ability, i.e., preserving silent features of self-supervised contrastive pre-trained models.
\item We theoretically formulate the feature suppression during supervised fine-tuning and prove the superiority of preserving silent features.
\item We propose a strong baseline STEP orthogonal to previous DG methods that leverage self-supervised contrastive learning models and provide richer pre-trained features to improve the generalization performance.
\end{itemize}

\section{Related Work}
\label{sec:relatedwork}

\textbf{Domain Generalization.} Domain Generalization aims to improve the model's generalization performance on inaccessible test domains utilizing a series of source domains.
The current DG methods can be generally divided into four categories: 1) Domain-invariant feature learning, i.e., learning invariant representations on source domains by adding regularization terms. DICA~\cite{pmlr-v28-muandet13} and CIDDG~\cite{li2018ciddg} align different distributions of source domains, while CORAL~\cite{sun2016deep} and MMD-AAE~\cite{Li_2018_CVPR_mmdaae} minimize statistical metrics. There are also causality-inspired methods~\cite{arjovsky2019invariant,pmlr-v139-mahajan21b,Lv_2022_CVPR_cirl}.  Instead of extracting invariant features, we use the richer silent features obtained by self-supervised learning to achieve better generalization.
2) Domain-specific feature learning, i.e., capturing domain-specific information strongly correlated with labels in individual training domains as a complement to domain-invariant features~\cite{NEURIPS2021_exploit_ds,specif_invar,lowrank}. The silent features can be viewed as specific features on test domains to some extent, but they cannot be well extracted and preserved by existing methods since they are not as relevant with labels in training.
3) Data manipulation. L2A-OT~\cite{zhou2020learning_L2A-OT} and DDAIG~\cite{zhou2020ddaig} use generative models to create novel samples, ~\cite{domainmixup,zhou2021mixstyle} generate samples based on mixup\cite{zhang2018mixup}, and~\cite{Xu_2021_CVPR_fourier} proposes a Fourier-based data augmentation strategy.
4) Learning strategies. DAEL~\cite{9540778dael}, EoA~\cite{arpit2021ensemble} use ensemble learning approaches, while~\cite{Mansilla_2021_ICCV_gradient} applies gradient surgery. MLDG~\cite{Li_Yang_Song_Hospedales_2018_mldg} and Metareg~\cite{NEURIPS2018_647bba34_metareg} employ meta-learning to simulate domain shifts.
RSC~\cite{rsc} iteratively drops the dominant features, but still extracts features based on their contribution to the classification of the training data.
Some previous works have also explored contrastive learning in DG, SelfReg~\cite{Kim_2021_ICCV_selfreg} incorporates a regularization for positive pairs alignment, and PCL~\cite{Yao_2022_CVPR_pcl} uses a proxy-based contrastive learning approach.
Different from them, we focus on the semantic information gained from the intra-class sample contrast.

\textbf{Test-time Adaptation.} Test-time adaptation focus on recovering specific features in the target domain using limited unlabeled data during inference time.~\cite{Dubey_2021_CVPR_adapt} constructs an domain-adaptive classifier, CoTTA~\cite{Wang_2022_CVPR_cadapt} incorporates continual learning and T3A~\cite{NEURIPS2021_testtime_adjust} adjust the classifier via the pseudo-prototype. Similarly, STEP concentrates on features that are discriminative on the target domain, but does not access the test data. It strives to preserve richer silent features during training in the hope that some of them are useful in the test distribution. STEP can be further combined with the TTA methods to enhance the generalization ability while maintaining the performance under similar distribution by excluding the redundant part of the preserved silent features.

\textbf{Self-supervised Contrastive Learning.} The key idea of contrastive learning is to push dissimilar inputs apart in the feature space while bringing similar inputs closer. Recently, self-supervised contrastive learning methods based on instance discrimination~\cite{Wu_2018_CVPR} pretext tasks have achieved promising results in computer vision. MoCo~\cite{He_2020_CVPR} introduces the momentum encoder, SimCLR~\cite{pmlr-v119-chen20j} applies projection heads to boost performance, and SwAV~\cite{NEURIPS2020_70feb62b} integrates self-supervised contrastive learning with clustering methods. Recent studies further propose self-supervised contrastive learning frameworks without the usage of negative samples~\cite{NEURIPS2020_f3ada80d_BYOL,Chen_2021_CVPR_simsiam,pmlr-v139-barlow}. Some related theoretical research also inspired our work.~\cite{NEURIPS2021_27debb43_graph} introduced graph to model self-supervised contrastive learning, and~\cite{pmlr-v119-wang20k_alignment_uniformity} identified two crucial properties of contrastive loss: alignment and uniformity, the latter of which motivated us to exploit the rich intra-class information preserved by self-supervised learning to enhance generalization performance.

\section{Problem Formulation}
\label{sec:formulation}
 Let $\cx$ and $\cy$ be an input space and an output space, respectively. Let $\ell:\cy\times\cy\to\mathbb{R}$ be a loss function. A \emph{domain} $\cd$ defines a joint distribution $\cd(\rvx,\rvy)$ over $\cx\times\cy$. During training, we have $(\rvx,\rvy)$ examples drawn from $n$ training domains $\cd_\rmtrain^1,\ldots,\cd_\rmtrain^n$. The aim of a DG algorithm is to learn a predictor, i.e., a parameterized function $h:\cx\to\cy$ that minimizes the expected risk $R_\rmtest(h) \vcentcolon= \mathbb{E}_{(\rvx,\rvy)\sim\cd_\rmtest}\ell[h(\rvx),\rvy]$ on a test domain $\cd_\rmtest\notin\{\cd_\rmtrain^1,\ldots,\cd_\rmtrain^n\}$ that is inaccessible during training. Throughout the paper, we focus on classification tasks with $\ell$ being the 0-1 loss defined as $\ell(\rvy,\rvy') \vcentcolon=\bbi_{\rvy\ne\rvy'}$. We assume that each model $h$ factorizes $h = g\circ \Phi$, where $\Phi:\cx\to\cz$ is a featurizer that maps each input to a feature space $\cz$ (e.g., the vector space corresponding to the output of the penultimate layer of a deep neural network), and $g:\cz\to\cy$ is a classifier on top of the feature.

\section{The Benefits of Preserving Silent Features}
\label{sec:theory}

In this section, we theoretically analyze the benefits of leveraging training-domain silent features in DG. We study the case where training and test data are i.i.d. drawn from two domains $\cd_\rmtrain$ and $\cd_\rmtest$, respectively. Note that here $\cd_\rmtrain$ can itself represent a distributional mixture of several training domains as in a common practical setup. For simplicity, we consider a binary classification problem with $\cy=\{1,-1\}$, but extending our results to the multi-class classification setting is straightforward.

It has been known that without any assumptions on the relation between $\cd_\rmtrain$ and $\cd_\rmtest$, DG is futile since $\cd_\rmtest$ can be chosen adversarially to yield a large risk given any model. Hence, to incorporate necessary structural assumptions on data, we adopt a Gaussian data generation model that is similar to the ones used in the DG literature~\cite{rosenfeld_risks_2021,chen_iterative_2021,wang_provable_2022}. Concretely, we assume that the data is generated through the following process, as depicted by Fig.~\ref{fig:dgm}:

\begin{enumerate}[noitemsep]
	\item A label $\rvy\in\{1, -1\}$ is drawn with a fixed probability:
	\begin{equation*}
	\rvy = \left\{\begin{array}{ll} 1, &\ \text{w.p.}\ \eta \\ -1, &\ \text{w.p.}\ 1 - \eta \end{array} \right.\label{eq:dgm_y}
	\end{equation*}
	\item Two types of features are both drawn according to a Gaussian. For $\mathcal{D}_\mathrm{train}$, we have
	\begin{equation*}
	\rvz_d\sim\mathcal{N}(\rvy\cdot\bm{\mu}_d, \sigma_d^2\bm{I}),\quad \rvz_s\sim\mathcal{N}(\rvy\cdot\bm{\mu}_s, \sigma_s^2\bm{I}),\label{eq:dgm_zd}
	\end{equation*}
	while for $\mathcal{D}_\mathrm{test}$, we have
	\begin{equation*}
	\rvz_d\sim\mathcal{N}(\rvy\cdot\bm{\mu}_d, \sigma_d^2\bm{I}),\quad \rvz_s\sim\mathcal{N}(\rvy\cdot\gamma \bm{\mu}_s, \sigma_s^2\bm{I}),\label{eq:dgm_zs}
	\end{equation*}
	where $\bm{\mu}_d\in \mathbb{R}^{p_d}$, $\bm{\mu}_s\in \mathbb{R}^{p_s}$, and $\gamma\in\mathbb{R},\gamma\ne 1$ is a scaling factor. Here, the \emph{dominant feature} $\rvz_d$ has identical distributions in training and test domains, while the \emph{silent feature} $\rvz_s$ may have different distributions in training and test domains depending on $\gamma$. We assume that there exists a constant $C>0$ and $\lVert\bm{\mu}_s\rVert_2^2 < C$ to capture our intuition that silent feature is not very discriminative in the training domain.
	\item The observed input $\rvx\in\mathcal{X}$ is generated by an injective function $f$: $\rvx = f(\rvz_d,\rvz_s)$.
\end{enumerate}
\begin{figure}[tbp]
\centering
\subcaptionbox{\label{subfig:train}Training domain}{\includegraphics[width=0.44\linewidth]{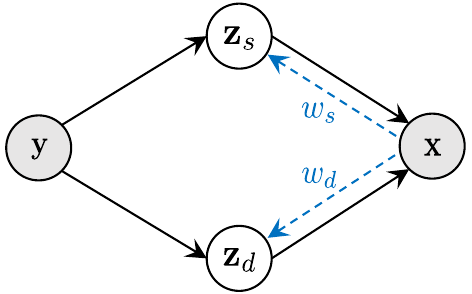}}
\hspace{1em}
\subcaptionbox{\label{subfig:test}Test domain}{\includegraphics[width=0.44\linewidth]{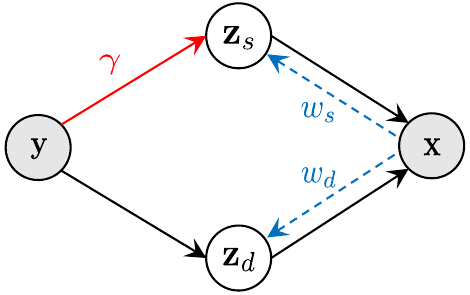}}
\caption{The causal graph of our data generation model in training and test domains. Shading represents that the variable is observed. Solid lines represent the data generation process, while dashed lines represent feature suppression.}
\label{fig:dgm}
\end{figure}
The above model provides a rigorous way to characterize the role of silent feature under different conditions: when $\gamma > 1$, silent feature is more predictive in $\cd_\rmtest$ than in $\cd_\rmtrain$; when $0<\gamma<1$, silent feature is less predictive in $\cd_\rmtest$ than in $\cd_\rmtrain$; when $\gamma < 0$, the correlation between silent feature and label is reversed in $\cd_\rmtest$. Based on this, previous DG methods that seek invariant representations~\cite{arjovsky2019invariant,sun2016deep,li2018domain,JMLR:v17:15-239_dann} can be viewed as finding a minimax-optimal predictor over all possible $\gamma$, boiling down to removing $\bm{\mu}_s$ completely. Meanwhile, if $C$ is sufficiently small, ERM also has the tendency to extract only $\bm{\mu}_d$ from $\rvx$ since $\bm{\mu}_s$ is not predictive enough. To formalize this, we introduce the notion of \emph{feature suppression}:

\begin{definition}[Feature suppression]
Given $w_d, w_s \in [0, 1]$, we say a featurizer $\Phi_{(w_d, w_s)}$ \emph{suppresses} $\rvz_d$ and $\rvz_s$ to $w_d$ and $w_s$ if $\Phi_{(w_d, w_s)}(\rvx) = (\widetilde{\rvz}_d, \widetilde{\rvz}_s)$, where
\begin{equation*}
\widetilde{\rvz}_d\sim\mathcal{N}(\rvy\cdot w_d\bm{\mu}_d, \sigma_d^2\bm{I}),\quad \widetilde{\rvz}_s\sim\mathcal{N}(\rvy\cdot w_s\bm{\mu}_s, \sigma_s^2\bm{I})\label{eq:suppress_train}
\end{equation*}
for $(\rvx,\rvy)\sim \mathcal{D}_\mathrm{train}$, and
\begin{equation*}
\widetilde{\rvz}_d\sim\mathcal{N}(\rvy\cdot w_d\bm{\mu}_d, \sigma_d^2\bm{I}),\quad \widetilde{\rvz}_s\sim\mathcal{N}(\rvy\cdot w_s\gamma \bm{\mu}_s, \sigma_s^2\bm{I})\label{eq:suppress_test}
\end{equation*}
for $(\rvx,\rvy)\sim \mathcal{D}_\mathrm{test}$.
\label{def:suppression}
\end{definition}

Under Definition~\ref{def:suppression}, prior methods that only uses dominant features elicit the featurizer $\Phi_{(1,0)}$ that suppresses silent feature to $w_s = 0$ while keeping dominant feature intact. However, we argue that this can be too conservative since we often do not expect silent features to act adversarially as $\gamma < 0$ in real-world datasets. In the following, we present a general result that characterizes the expected test domain risk $R_\rmtest$ \emph{as a function of $w_s, w_d$ and $\gamma$}.

\begin{theorem}[Expected test domain risk]
Assume that $\eta = \frac{1}{2}$ and $\sigma_d^2=\sigma_s^2=\sigma^2$. Then,
for any $w_d,w_s\in[0,1]$, the predictor $g^*\circ\Phi_{(w_d,w_s)}$, composed of featurizer $\Phi_{(w_d,w_s)}$ and training-domain Bayes classifier $g^*$ with respect to $\Phi_{(w_d,w_s)}$, gives the expected test domain risk of
\begin{equation*}
R_\rmtest\left(g^*\circ\Phi_{(w_d,w_s)}\right) = F\left(-\frac{1}{\sigma}\cdot\frac{{w}_d^2\lVert\bm{\mu}_d\rVert_2^2 + \gamma{w}_s^2\lVert\bm{\mu}_s\rVert_2^2}{\sqrt{ {w}_d^2\lVert\bm{\mu}_d\rVert_2^2 +{w}_s^2 \lVert \bm{\mu}_s \rVert_2^2}}\right),
\end{equation*}
where $F$ is the CDF of a standard Gaussian $\mathcal{N}(0,1)$.
\label{theo:error}
\end{theorem}
Note that the assumptions on $\eta$, $\sigma_d^2$, and $\sigma_s^2$ are purely for the clarity of presentation, and the full theorem is presented in Appendix~\ref{appendix:proofs} along with the proof. Theo.~\ref{theo:error} shows that when $\gamma < 0$, the invariant featurizer $\Phi_{(1,0)}$ is indeed optimal; however, in other cases explicitly using silent feature by increasing $w_s$ may outperform the invariant featurizer, depending on the more nuanced relations between $\bm{\mu}_s,\bm{\mu}_d$ and $\gamma$. We provide detailed discussion in Appendix~\ref{appendix:proofs}.

\section{Method}

\subsection{Motivation}
\label{subsec:motivation}

Before presenting the details of our method, we first clarify our motivation. Previous research has shown that self-supervised contrastive pre-trained models~\cite{chen2020improved_moco_v2,pmlr-v119-chen20j,NEURIPS2020_70feb62b} have better robustness on downstream tasks. However, we empirically find that directly employing the model pre-trained on ImageNet via self-supervised contrastive learning in the DG setting does not obtain higher generalization performance than the conventional supervised learning pre-trained model. As shown in~\cref{tab:motivation}, replacing the supervised learning pre-trained model in ERM with a self-supervised contrastive learning pre-trained model (SwAV~\cite{NEURIPS2020_70feb62b}) resulted in a 1.2 percentage point(pp) decrease in the average test domain accuracies on five DG benchmarks (both using ResNet-50~\cite{He_2016_CVPR_resnet} trained on ImageNet for a fair comparison). We conjecture that this is caused by feature suppression as in~\cref{def:suppression}. Compared to supervised learning, self-supervised contrastive learning introduces more intra-class discriminative features during negative sample contrast within individual classes. We refer to these features as silent features since they have little classification contribution on the training domains and are therefore suppressed during the ERM's supervised fine-tuning procedure. Silent features may be more discriminative on the target domain with significant distribution shifts, so the suppression of silent features makes the self-supervised contrastive learning pre-trained model perform poorly on the DG benchmarks. According to~\cref{theo:error}, keeping silent features, in this case, yield a smaller expected test domain risk.
Hence, we can preserve more silent features and promote OOD generalization at the sacrifice of little i.i.d generalization performance, according to our theoretical analysis in \cref{appendix:proofs} and some prior works discussing the trade-off between i.i.d. and OOD generalization abilities~\cite{pmlr-v119-raghunathan20a_tradeoff1,xie2021innout_tradeoff2}.

\noindent
\textbf{The relative weight of silent features in pre-trained models.} In fact, the downstream DG datasets and the ImageNet pre-training classes are different, and thus there are also existing silent features extracted from redundant categories in the supervised learning pre-trained model. Of course, silent features extracted from intra-class information are richer in the self-supervised contrastive learning pre-trained model, so the relative weight of silent features is higher. In addition, different self-supervised contrastive learning algorithms provide various weights of silent features in the pre-trained models, which we will further discuss in~\cref{subsec:ablation}. We use SwAV as a demo in this paper.

\begin{table*}
  \caption{Comparison of self-supervised contrastive learning and supervised learning models pre-trained on ImageNet.}
  \centering
  \begin{tabular}{l|ccccc|c}
    \toprule
    Methods & VLCS & PACS & Office-Home & TerraInc & DomainNet & Avg.($\Delta$)\\
    \midrule
    Supervised+ERM & \textbf{77.9 \bm{$\pm$} 0.4} & 84.6 $\pm$ 0.5 & \textbf{66.6 \bm{$\pm$} 0.7} & \textbf{48.6 \bm{$\pm$} 3.0} & 42.3 $\pm$ 0.2 & \textbf{64.0} \\
    SwAV~\cite{NEURIPS2020_70feb62b}+ERM & 77.6 $\pm$ 0.7 & \textbf{84.7 \bm{$\pm$} 1.5} & 63.1 $\pm$ 0.3 & 46.1 $\pm$ 2.9 & \textbf{42.4 \bm{$\pm$} 0.0} & 62.8(-1.2) \\
    \bottomrule
  \end{tabular}
  \label{tab:motivation}
\end{table*}

We provide an intuitive 2-dimensional example to further elaborate on the effect of increasing the relative weight of silent features on the OOD generalization performance. As shown in~\cref{fig:instance-level}, the self-supervised contrastive pre-trained model with richer silent features increases the weight of shape in distinguishing between images of the same class (e.g., St. Bernard dogs of diﬀerent body shapes and poses), and gains higher classification performance on the distribution of the ``sketch" target domain where shape feature is more discriminative. The supervised learning pre-trained model, as shown in~\cref{fig:category-level}, achieves robust classification between dog and elephant in the ``photo" domain via abundant texture information since no further intra-class contrast is required.
Of course, the improved OOD generalization performance comes at the expense of a part of the i.i.d. performance, as~\cref{fig:simple_example} shows that the distance in new feature space between dogs and elephants in the ``photo" domain becomes lower compared with before. Therefore, for future work, we need to balance the relative weight of silent and dominant features, depending on the scenario and potential degree of distribution shift.

\begin{figure*}
  \centering
  \begin{subfigure}{0.495\linewidth}
    \includegraphics[width=1.0\linewidth]{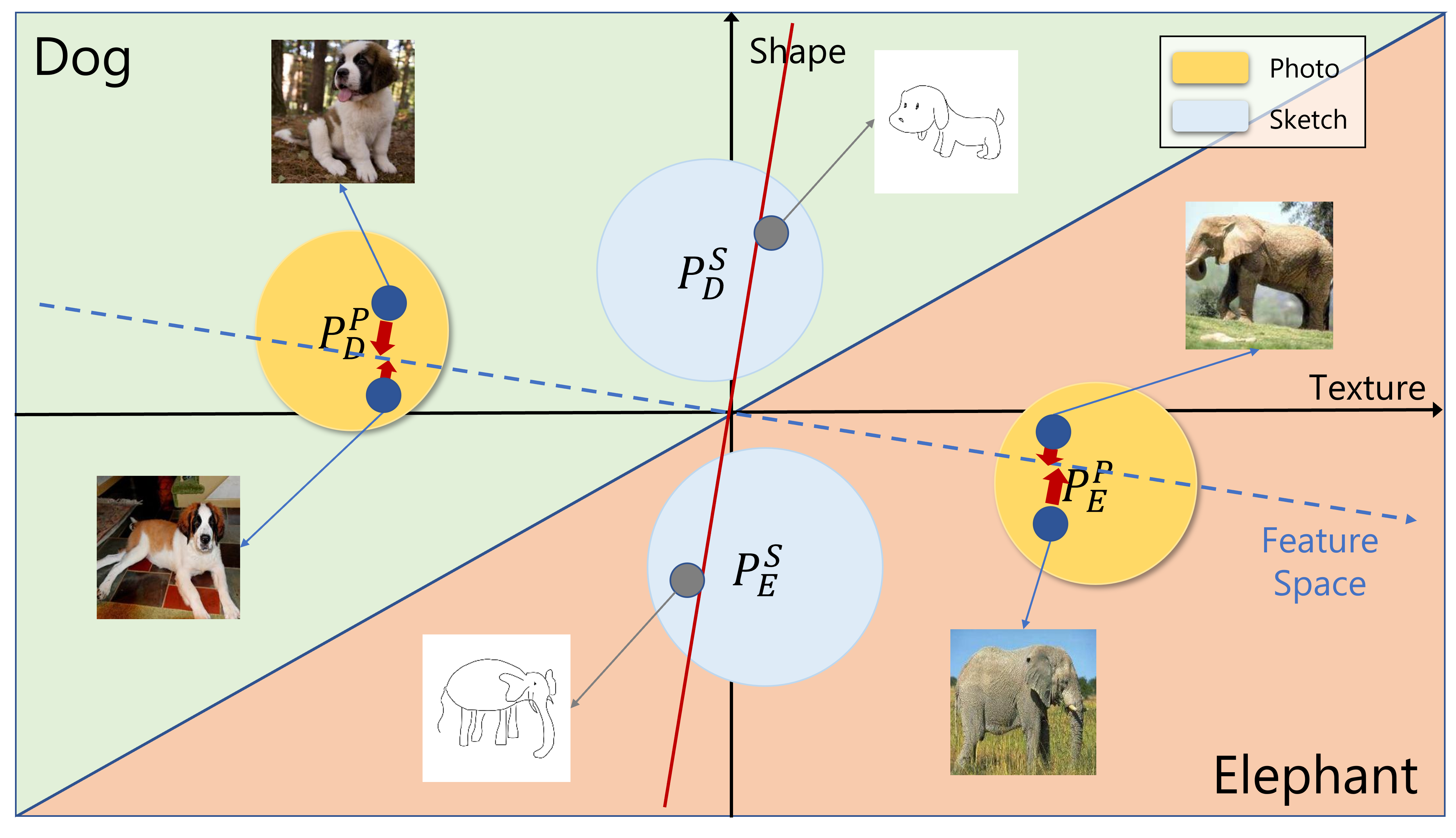}
    \caption{Supervised pre-trained model}
    \label{fig:category-level}
  \end{subfigure}
  \hfill
  \begin{subfigure}{0.495\linewidth}
    \includegraphics[width=1.0\linewidth]{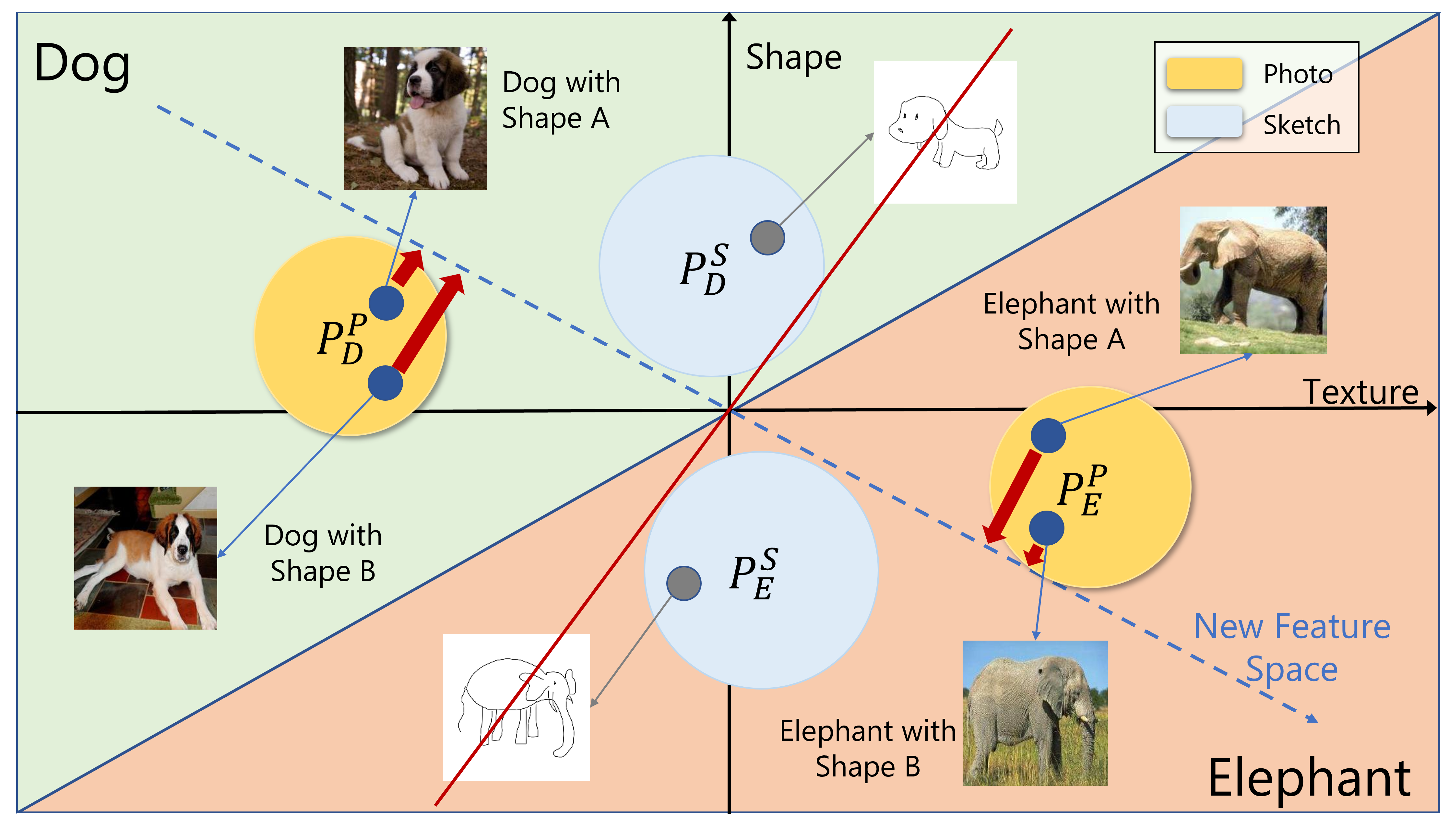}
    \caption{Self-supervised contrastive pre-trained model}
    \label{fig:instance-level}
  \end{subfigure}
  \caption{Example of the correlation between OOD generalization performance and the relative weight of silent features, images are from the PACS dataset~\cite{Li_2017_ICCV_pacs}. The source domain is the ``photo" domain in yellow, while the target domain is the ``sketch" domain in gray-blue, which distinguishes well in the texture and shape dimensions respectively. The diagonal line represents the ideal decision boundary for the dog and the elephant, the blue dashed line is the feature space given by mixing the two feature dimensions, and the red line is the corresponding empirical decision boundary. It is evident that the decision boundary fitted by the Self-supervised contrastive learning pre-trained model on the ``photo" domain is closer to the ideal situation and generalizes better on the inaccessible ``sketch" domain.}
  \label{fig:simple_example}
\end{figure*}

\subsection{Silent Feature Preservation}

In this section, we will go into depth about our proposed simple yet effective method STEP for preserving silent features and improving the stability of generalization performance, which consists of two components on top of the pre-trained models with silent features: the LP-FT~\cite{kumar2022finetuning} strategy, and the SWAD~\cite{NEURIPS2021_swad} technique. Note that our method is orthogonal to previous DG algorithms, which means STEP can be a plug-in module when one wants to exploit the silent features extracted by self-supervised contrastive learning. We will demonstrate the complementarity of STEP with existing DG methods in~\cref{subsec:ablation}.

\subsubsection{Linear Probing then fine-tuning (LP-FT)}

Based on our previous analysis, supervised learning has a considerable suppression effect on the silent features related to intra-class discrimination. Nonetheless, the majority of the existing DG methods employ supervised fine-tuning on the downstream DG datasets~\cite{gulrajani2021in,Ye_2022_CVPR_oodbench}, which will inevitably result in the re-suppression of silent features extracted by the self-supervised contrastive learning algorithms with intra-class negative pair comparison. To address this issue, we introduce a two-stage learning strategy of linear probing followed by overall fine-tuning.

In the first stage, we freeze the parameters of the pre-trained backbones and train the linear classification head only. This stage will directly inherit the pre-trained features, allowing the model to initially adapt to the downstream datasets and optimize the loss to compress the subsequent parameter search space. In the second stage, we fine-tune the full parameters of the model so that it can better transfer to the specific classification task. With the two-stage learning strategy, we minimize the impact of the supervised fine-tuning process on the silent features and preserve more intra-class information.

Additionally, to further support our theoretical analysis in~\cref{sec:theory}, we design an experiment, that trains the backbones with only linear probing on standard DG datasets, to show the superiority of silent features' generalization potential on target data with significant distribution shifts. The detailed results are shown in~\cref{appendix:supp_experiment}.

\subsubsection{Stochastic Weight Averaging Densely (SWAD)}

In addition to the feature suppression phenomenon mentioned in~\cref{sec:intro}, we also observed significant fluctuations in the test error after introducing the self-supervised contrastive learning pre-trained model during the actual training process. This noticeably affects the stability of the model's generalization performance on the target domain.
We conjecture that this is induced by the weak correlation between silent features and the given i.i.d. distribution on the classification task, which leads to a more complex solution space. As a result, incorporating more diverse silent features prevents the model from readily converging to a solution with better generalization performance.

SWAD is an ensemble learning method, that incorporates an overfit-aware mechanism and averaging the model weights over the training iterations, to help the model converge to flatter minima. As flat minima increase the robustness of the model~\cite{hochreiter1997flat,keskar2017flat2} and have a smaller generalization gap under the DG scenario~\cite{NEURIPS2021_swad}, we adopt the SWAD technique in the second stage (fine-tuning) to improve the convergence stability and generalization performance of the pre-trained backbones with self-supervised contrastive learning.

\section{Experiment}

\subsection{Experimental Settings}

\textbf{Datasets.} We validate our method on five standard DG datasets for object recognition, and the details of each dataset are provided below: (1) \textbf{VLCS}~\cite{Fang_2013_ICCV_vlcs}: consists of four domains (\textit{Caltech101, PASCAL VOC, LabelMe, SUN09}) that were captured with various cameras and viewpoints at different times and locations. It contains 5 categories and 10,729 images in total. We adopt the data partition of~\cite{Carlucci_2019_CVPR_jigsaw}. (2) \textbf{PACS}~\cite{Li_2017_ICCV_pacs}: consists of 4 domains of different styles (\textit{Art Painting, Cartoon, Photo, Sketch}) with larger distribution shifts. Each domain contains 7 categories with a total of 9,991 images. We use the original train-val split from~\cite{Li_2017_ICCV_pacs}. (3) \textbf{Office-Home}~\cite{Venkateswara_2017_officehome}: contains 15,500 images in 65 categories from 4 domains (\textit{Art, Clipart, Product, Real-World}) with various styles and backgrounds. Since Office-Home and ImageNet distributions are similar and have relatively smaller diversity shifts~\cite{Ye_2022_CVPR_oodbench}, we use this dataset to evaluate the performance of silent features on the test domain similar to the training distributions. (4) \textbf{Terra Incognita}~\cite{Beery_2018_ECCV_terra}: comprises 24,788 images in 10 categories of 4 domains (\textit{L100, L38, L43, L46}) from camera traps placed in different locations, with domain shifts primarily brought by environment and camera parameters. (5) \textbf{DomainNet}~\cite{Peng_2019_ICCV_domainnet}: a large-scale DG dataset with 586,575 images of 345 categories from 6 domains in various styles (\textit{clipart, infograph, painting, quickdraw, real, sketch}). We follow the original data split provided in~\cite{Peng_2019_ICCV_domainnet}.

\textbf{Baselines.} We compared STEP with 16 baselines: ERM~\cite{vapnik1999nature}, IRM~\cite{arjovsky2019invariant}, GroupDRO~\cite{Sagawa*2020Distributionally_dro}, Mixup~\cite{yan2020improve_mixup_domainbed}, MLDG~\cite{Li_Yang_Song_Hospedales_2018_mldg}, CORAL~\cite{sun2016deep}, MMD~\cite{Li_2018_CVPR_mmdaae}, DANN~\cite{JMLR:v17:15-239_dann}, CDANN~\cite{li2018ciddg}, MTL~\cite{blanchard2021domain_mtl}, SagNet~\cite{Nam_2021_CVPR_sagnet}, ARM~\cite{NEURIPS2021_c705112d_arm}, VREx~\cite{pmlr-v139-krueger21a_vrex}, RSC~\cite{rsc}, SelfReg~\cite{Kim_2021_ICCV_selfreg}, SWAD~\cite{NEURIPS2021_swad}. 

Most of the baselines are from the reported results of the DomainBed~\cite{gulrajani2021in}, except for SWAD (also built on DomainBed).

\begin{table*}
  \caption{Comparison with DG algorithms and STEP on five DG benchmarks. We highlight the \textbf{best} and \underline{second-best} results in bold and underlined respectively. Note that ERM (our runs) and SWAD (our runs) are reproduced results under our hyperparameter search protocol, baselines from the reports in DomainBed~\cite{gulrajani2021in} are denoted with $\dagger$, and other results denoted with $\ast$ are from the original literature. We selected the results based on the \textbf{oracle} model selection criterion in DomainBed to ensure a fair comparison. STEP-S refers to SwAV + STEP.}
  \centering
  \begin{tabular}{l|ccccc|c|c}
    \toprule
    Methods & VLCS & PACS & Office-Home & TerraInc & DomainNet & Avg. & Avg. w/o O-H\\
    \midrule
    ERM$^{\dagger}$~\cite{vapnik1999nature} & 77.6$\pm$0.3 & 86.7$\pm$0.3 & 66.4$\pm$0.5 & 53.0$\pm$0.3 & 41.3$\pm$0.1 & 65.0 & 64.7 \\
    ERM(reproduced) & 77.9$\pm$0.4 & 84.6$\pm$0.5 & 66.6$\pm$0.7 & 48.6$\pm$3.0 & 42.3$\pm$0.2 & 64.0 & 63.4 \\
    IRM$^{\dagger}$~\cite{arjovsky2019invariant} & 76.9$\pm$0.6 & 84.5$\pm$1.1 & 63.0$\pm$2.7 & 50.5$\pm$0.7 & 28.0$\pm$5.1 & 60.6 & 60.0 \\
    GroupDRO$^{\dagger}$~\cite{Sagawa*2020Distributionally_dro} & 77.4$\pm$0.5 & 87.1$\pm$0.1 & 66.2$\pm$0.6 & 52.4$\pm$0.1 & 33.4$\pm$0.3 & 63.3 & 62.6 \\
    Mixup$^{\dagger}$~\cite{domainmixup} &78.1$\pm$0.3 & 86.8$\pm$0.3 & 68.0$\pm$0.2 & \underline{54.4$\pm$0.3} & 39.6$\pm$0.1 & 65.4 & 64.7 \\
    MLDG$^{\dagger}$~\cite{Li_Yang_Song_Hospedales_2018_mldg}  &77.5$\pm$0.1 & 86.8$\pm$0.4 & 66.6$\pm$0.3 & 52.0$\pm$0.1 & 41.6$\pm$0.1 & 64.9 & 64.5 \\
    CORAL$^{\dagger}$~\cite{sun2016deep} & 77.7$\pm$0.2 & 87.1$\pm$0.5 & 68.4$\pm$0.2 & 52.8$\pm$0.2 & 41.8$\pm$0.1 & 65.6 & 64.9 \\
    MMD$^{\dagger}$~\cite{Li_2018_CVPR_mmdaae} & 77.9$\pm$0.1 & 87.2$\pm$0.1 & 66.2$\pm$0.3 & 52.0$\pm$0.4 & 23.5$\pm$9.4 & 61.4 & 60.2 \\
    DANN$^{\dagger}$~\cite{JMLR:v17:15-239_dann} & \underline{79.7$\pm$0.5} & 85.2$\pm$0.2 & 65.3$\pm$0.8 & 50.6$\pm$0.4 & 38.3$\pm$0.1 & 63.8 & 63.5 \\
    CDANN$^{\dagger}$~\cite{li2018ciddg} & \textbf{79.9\bm{$\pm$}0.2} & 85.8$\pm$0.8 & 65.3$\pm$0.5 & 50.8$\pm$0.6 & 38.5$\pm$0.2 & 64.1 & 63.8 \\
    MTL$^{\dagger}$~\cite{blanchard2021domain_mtl} & 77.7$\pm$0.5 & 86.7$\pm$0.2 & 66.5$\pm$0.4 & 52.2$\pm$0.4 & 40.8$\pm$0.1 & 64.8 & 64.4 \\
    SagNet$^{\dagger}$~\cite{Nam_2021_CVPR_sagnet} & 77.6$\pm$0.1 & 86.4$\pm$0.4 & 67.5$\pm$0.2 & 52.5$\pm$0.4 & 40.8$\pm$0.2 & 65.0 & 64.3 \\
    ARM$^{\dagger}$~\cite{NEURIPS2021_c705112d_arm} & 77.8$\pm$0.3 & 85.8$\pm$ 0.2& 64.8$\pm$0.4 & 51.2$\pm$0.5 & 36.0$\pm$0.2 & 63.1 & 62.7 \\
    VREx$^{\dagger}$~\cite{pmlr-v139-krueger21a_vrex} & 78.1$\pm$0.2 & 87.2$\pm$0.6 & 65.7$\pm$0.3 & 51.4$\pm$0.5 & 30.1$\pm$3.7 & 62.5 & 61.7 \\
    RSC$^{\dagger}$~\cite{rsc}  & 77.8$\pm$0.6 & 86.2$\pm$0.5 & 66.5$\pm$0.6 & 52.1$\pm$0.2 & 38.9$\pm$0.6 & 64.3 & 63.8 \\
    SelfReg$^{\ast}$~\cite{Kim_2021_ICCV_selfreg} & 77.5$\pm$0.0 & 86.5$\pm$0.3 & 69.4$\pm$0.2 & 51.0$\pm$0.4 & 44.6$\pm$0.1 & 65.8 & 64.9 \\ 
    SWAD$^{\ast}$~\cite{NEURIPS2021_swad} & 79.1$\pm$0.1 & \underline{88.1$\pm$0.1} & \textbf{70.6\bm{$\pm$}0.2} & 50.0$\pm$0.3 & \underline{46.5$\pm$0.1} & \underline{66.9} & \underline{65.9} \\
    SWAD(reproduced) & 78.1$\pm$0.4 & 87.1$\pm$0.1 & \textbf{70.6\bm{$\pm$}0.3} & 50.6$\pm$0.7 & 45.9$\pm$0.0 & 66.5 & 65.4 \\
    \midrule
    STEP-S(ours) & 78.6$\pm$0.4 & \textbf{88.7\bm{$\pm$}0.4} & \underline{68.6$\pm$0.1} & \textbf{55.6\bm{$\pm$}0.1} & \textbf{47.3\bm{$\pm$}0.0} & \textbf{67.8} & \textbf{67.6}\\
    \bottomrule
  \end{tabular}
  \label{tab:STEP}
\end{table*}

\textbf{Hyperparameter search.} To reduce the computational cost of random search over the hyperparameter distribution in DomainBed, we use a discrete hyperparameter search space and fix low-impact hyperparameters like weight decay, batch size, and dropout rate. For each hyperparameter, we select the model with the highest accuracy on the training-domain validation set. Then we conduct a grid search on the designed hyperparameter space to identify the best-performing hyperparameters on the test-domain validation set. The hyperparameter search space of all the experiments as well as the optimal hyperparameters are shown in the~\cref{appendix:experiment_settings}.

\textbf{Implementation details.} We generally follow the training and evaluation protocols of DomainBed and SWAD, including model selection on the training-domain validation set, Adam~\cite{DBLP:journals/corr/KingmaB14_adam} optimizer, and network architectures, but we perform hyperparameter search on the test-domain validation set. As for dataset partitioning, we choose to maintain the original data split if possible, otherwise, we randomly split the training and test sets in a ratio of 9:1 and kept the split fixed throughout the experiment.We increase the number of iterations for DomainNet to 15,000 for the main experiment while maintaining 5000 iterations for the other datasets, the same as in SWAD. Furthermore, to better stabilize the training fluctuations of STEP, we set a lower tolerance rate r and evaluation frequency compared with SWAD. Due to the above differences and the reduced hyperparameter search space, we \textbf{reproduce ERM and SWAD under our experimental settings}. For a fair comparison, we ensure that in each experiment the total number of iterations for linear probing and fine-tuning of STEP is not more than that of SWAD. Meanwhile, we use ResNet-50~\cite{He_2016_CVPR_resnet} pre-trained on ImageNet~\cite{deng2009imagenet} based on SwAV~\cite{NEURIPS2020_70feb62b} as the initial weight for STEP to make sure that no other data are introduced. More implementation details are given in the~\cref{appendix:experiment_settings}.

\textbf{Evaluation metrics.} Each time, we choose a domain as the unseen target domain, and we select the model with the highest average target domain accuracy. We repeat the entire experiment three times with three \textbf{fixed random seeds} to ensure the reliability and reproducibility of the experimental results. Finally, we report the mean and standard deviation of the average out-of-domain accuracies across all target domains over three repetitions.

\subsection{Main Experiment Results}

We conduct comprehensive trials on five DG benchmarks, the results of which are provided in~\cref{tab:STEP}, to confirm the superiority of preserving silent features on DG tasks. We can observe that STEP-S (SwAV+STEP) achieves the highest accuracies on three of the four benchmarks with larger distribution shifts, while also getting the second-best result on Office-Home which has a lower diversity shift~\cite{Ye_2022_CVPR_oodbench}.
In terms of average accuracy across all benchmarks, STEP-S outperforms all baselines, improving by 3.8 and 1.3 pp above the ERM and SWAD reproduced under our settings, respectively. It should be emphasized that we use a \textbf{reduced search space} in our trials to save computational costs. STEP shows state-of-the-art performances on four DG benchmarks apart from Office-Home, with a \textbf{large margin} (\textbf{+4.2/+2.2 pp} v.s. reproduced ERM/SWAD) on average out-of-domain accuracy, while most of the baselines improve the performance of ERM by less than 1.0pp. Furthermore, STEP is \textbf{orthogonal} to most of the prior DG approaches, allowing its integration with other DG methods to obtain better generalization performance with richer silent features. We give the full results in~\cref{appendix:results}.

\begin{table*}
  \caption{Comprehensive ablation studies on all five benchmarks. Results denoted with $\dagger$ are from DomainBed (oracle).}
  \centering
  \begin{tabular}{l|ccccc|c}
    \toprule
    Variants & VLCS & PACS & Office-Home & TerraInc & DomainNet & Avg.($\Delta$)\\
    \midrule
    Supervised + ERM & 77.9 $\pm$ 0.4 & 84.6 $\pm$ 0.5 & 66.6 $\pm$ 0.7 & 48.6 $\pm$ 3.0 & 42.3 $\pm$ 0.2 & 64.0 \\
    \midrule
    Supervised + LP-FT & 77.7 $\pm$ 0.2 & 83.3 $\pm$ 0.9 & 67.7 $\pm$ 0.4 & 49.2 $\pm$ 2.0 & 42.6 $\pm$ 0.2 & 64.1(+0.1) \\
    SwAV + LP-FT & \textbf{79.1 $\bm{\pm}$ 0.5} & 86.0 $\pm$ 0.3 & 63.8 $\pm$ 0.1 & 52.4 $\pm$ 1.2 & 43.0 $\pm$ 0.1 & 64.9(+0.9) \\
    \midrule
    Supervised + SWAD & 78.1 $\pm$ 0.4 & 87.1 $\pm$ 0.1 & 70.6 $\pm$ 0.3 & 50.6 $\pm$ 0.7 & 45.9 $\pm$ 0.0 & 66.5(+2.5) \\
    SwAV + SWAD & 77.3 $\pm$ 0.4 & 88.0 $\pm$ 0.4 & 68.4 $\pm$ 0.1 & 52.0 $\pm$ 0.7 & 45.9 $\pm$ 0.1 & 66.3(+2.3) \\
    \midrule
    Supervised + STEP & 78.0 $\pm$ 0.3 & 86.9 $\pm$ 0.5 & \textbf{70.8 \bm{$\pm$} 0.1} & 51.2 $\pm$ 0.3 & 46.1 $\pm$ 0.0 & 66.6(+2.6) \\
    MoCo v2~\cite{chen2020improved_moco_v2} + STEP  & 79.0 $\pm$ 0.1 & 89.2 $\pm$ 0.0 & 67.7 $\pm$ 0.2 & 53.6 $\pm$ 0.7 & 46.1 $\pm$ 0.1 & 67.1(+3.1) \\
    Barlow Twins~\cite{pmlr-v139-barlow} + STEP & 77.5 $\pm$ 0.2 & 86.6 $\pm$ 0.3 & 65.7 $\pm$ 0.2 & 50.3 $\pm$ 0.5 & 45.5 $\pm$ 0.0 & 65.1(+1.1) \\
    \midrule
    CORAL$^{\dagger}$ & 77.7 $\pm$ 0.2 & 87.1 $\pm$ 0.5 & 68.4 $\pm$ 0.2 & 52.8 $\pm$ 0.2 & 41.8 $\pm$ 0.1 & 65.6(+1.6) \\
    CORAL + STEP-S & 78.7 $\pm$ 0.2 & \textbf{88.9 $\bm{\pm}$ 0.3} & \textbf{70.8 $\bm{\pm}$ 0.2} & \textbf{55.4 $\bm{\pm}$ 0.5} & \textbf{47.6 $\bm{\pm}$ 0.0} & \textbf{68.3(+4.3)} \\
    \bottomrule
  \end{tabular}
  \label{tab:ablation}
\end{table*}

\subsection{Ablation Study}
\label{subsec:ablation}

\noindent \textbf{Preserving silent features plays a key role in boosting performance.} As shown in~\cref{tab:motivation} and~\cref{tab:ablation}, for SwAV, both the direct ERM and the weight-ensemble method SWAD yield lower outcomes than the supervised learning baselines.  Only the addition of LP-FT, which modifies the prior supervised fine-tuning procedure, improves the DG performance of SwAV (+2.1pp on average compared to SwAV + ERM). Furthermore, LP-FT does not enhance the performance of the supervised learning pre-trained model, which implies that the relative weight of silent features is lower, demonstrating the suppression of intra-class silent features by the supervised objective. We also conducted supplementary experiments on the Euclidean distance of the features before and after training, as well as the mutual information between raw inputs and features, and the experimental results support our conjectures. The thorough analysis of the supplementary experiments is presented in~\cref{appendix:supp_experiment}.

\noindent \textbf{Impact of different pre-training algorithms.} To get a full picture of the relationship between STEP's performance and the relative weights of the silent features of the pre-trained models, we tested the supervised model, MoCo v2~\cite{chen2020improved_moco_v2}, and Barlow Twins~\cite{pmlr-v139-barlow} trained on ImageNet in addition to SwAV. Table 3 shows that STEP does not improve the supervised pre-trained model compared to SWAD. On the other hand, Barlow Twins, which does not use negative pairs, even has an average decrease of 1.5 pp compared to the supervised model. Only SwAV and MoCo v2, which use intra-class negative sample contrasts, exhibit a clear improvement over SWAD, indicating that silent features with good generalization potential over unseen target domains came more from intra-class contrast information. Furthermore, MoCo v2's decreased effect, which treats all augmented views of different samples as negative pairs, displays a trade-off between the relative weights of silent features and model generalization performance. It shows that the exaggerated weight of silent features impairs generalization ability by aggravating feature redundancy.

\noindent \textbf{Silent features perform better with significant  distribution shifts.} As seen in~\cref{tab:motivation} and~\cref{tab:ablation}, the introduction of SwAV, a self-supervised pre-trained model containing more intra-class silent features, performs significantly weaker than the supervised learning pre-trained model on Office-Home, which has a lower distribution shift. The average performance improvement of STEP-S on datasets other than Office-Home is more significant. Compared to reproduced ERM and SWAD baselines improved by 4.2 and 2.2 pp, respectively. This is because when the diversity shift is insignificant, the dominant features on the training and test distributions are more similar, resulting in the poor discriminability of silent features on the target domain. It is worth mentioning that STEP mitigates this phenomenon, narrowing the gap with the supervised model by 1.5pp.

\noindent \textbf{Complementarity of STEP and existing DG methods.}~\cref{tab:ablation} shows the combination with CORAL further compensates for STEP's weakness on datasets with small diversity shifts, achieving \textbf{new state-of-the-art} performance.

\section{Conclustion}

\noindent \textbf{Limitations.} We do not provide an explicit definition for silent features, and STEP is not perfect for preserving silent features, the second stage of fine-tuning still corrupts the original pre-trained features to some extent. In addition, STEP might cause feature redundancy issues since it's unclear which part of the silent features generalizes well across unseen domains.

In this paper, we introduce a new perspective for improving generalization performance on the preservation of silent features in domain generalization. The key idea is that silent features with rich intra-class contrast information, which are suppressed during supervised fine-tuning, have superior generalization ability over data with significant distribution discrepancies. We theoretically formulate the feature suppression phenomenon and analyze the benefits of preserving silent features for improving the generalization ability to unseen distributions. Extensive experiments and ablations demonstrate the state-of-the-art performance of our proposed STEP method on DG tasks and elucidate the impact of different pre-trained features and the applicable scenarios. Considering that the current mainstream DG methods use only supervised pre-trained features, we hope our work can serve as an exploratory first step to improve DG performance via effectively preserving intra-class silent features of the pre-trained model and shed some light on the DG community.

{\small
\bibliographystyle{ieee_fullname}
\bibliography{egbib}

\begin{thebibliography}{10}\itemsep=-1pt

\bibitem{arjovsky2019invariant}
Martin Arjovsky, L{\'e}on Bottou, Ishaan Gulrajani, and David Lopez-Paz.
\newblock Invariant risk minimization.
\newblock {\em arXiv preprint arXiv:1907.02893}, 2019.

\bibitem{arpit2021ensemble}
Devansh Arpit, Huan Wang, Yingbo Zhou, and Caiming Xiong.
\newblock Ensemble of averages: Improving model selection and boosting
  performance in domain generalization.
\newblock In {\em NeurIPS}, 2022.

\bibitem{NEURIPS2018_647bba34_metareg}
Yogesh Balaji, Swami Sankaranarayanan, and Rama Chellappa.
\newblock Metareg: Towards domain generalization using meta-regularization.
\newblock In {\em NeurIPS}, 2018.

\bibitem{Beery_2018_ECCV_terra}
Sara Beery, Grant Van~Horn, and Pietro Perona.
\newblock Recognition in terra incognita.
\newblock In {\em ECCV}, 2018.

\bibitem{blanchard2021domain_mtl}
Gilles Blanchard, Aniket~Anand Deshmukh, {\"U}run Dogan, Gyemin Lee, and
  Clayton Scott.
\newblock Domain generalization by marginal transfer learning.
\newblock {\em Journal of Machine Learning Research}, 22(1):46--100, 2021.

\bibitem{NEURIPS2021_exploit_ds}
Manh-Ha Bui, Toan Tran, Anh Tran, and Dinh Phung.
\newblock Exploiting domain-specific features to enhance domain generalization.
\newblock In {\em NeurIPS}, pages 21189--21201, 2021.

\bibitem{Carlucci_2019_CVPR_jigsaw}
Fabio~M. Carlucci, Antonio D'Innocente, Silvia Bucci, Barbara Caputo, and
  Tatiana Tommasi.
\newblock Domain generalization by solving jigsaw puzzles.
\newblock In {\em CVPR}, 2019.

\bibitem{NEURIPS2020_70feb62b}
Mathilde Caron, Ishan Misra, Julien Mairal, Priya Goyal, Piotr Bojanowski, and
  Armand Joulin.
\newblock Unsupervised learning of visual features by contrasting cluster
  assignments.
\newblock In {\em NeurIPS}, pages 9912--9924, 2020.

\bibitem{NEURIPS2021_swad}
Junbum Cha, Sanghyuk Chun, Kyungjae Lee, Han-Cheol Cho, Seunghyun Park, Yunsung
  Lee, and Sungrae Park.
\newblock Swad: Domain generalization by seeking flat minima.
\newblock In {\em NeurIPS}, pages 22405--22418, 2021.

\bibitem{miro}
Junbum Cha, Kyungjae Lee, Sungrae Park, and Sanghyuk Chun.
\newblock Domain generalization by mutual-information regularization
  with pre-trained models.
\newblock In {\em ECCV}, pages 440--457, 2022.

\bibitem{specif_invar}
Prithvijit Chattopadhyay, Yogesh Balaji, and Judy Hoffman.
\newblock Learning to balance specificity and invariance for in and out of
  domain generalization.
\newblock In {\em ECCV}, pages 301--318, 2020.

\bibitem{pmlr-v119-chen20j}
Ting Chen, Simon Kornblith, Mohammad Norouzi, and Geoffrey Hinton.
\newblock A simple framework for contrastive learning of visual
  representations.
\newblock In {\em ICML}, pages 1597--1607, 2020.

\bibitem{chen2020improved_moco_v2}
Xinlei Chen, Haoqi Fan, Ross Girshick, and Kaiming He.
\newblock Improved baselines with momentum contrastive learning.
\newblock {\em arXiv preprint arXiv:2003.04297}, 2020.

\bibitem{Chen_2021_CVPR_simsiam}
Xinlei Chen and Kaiming He.
\newblock Exploring simple siamese representation learning.
\newblock In {\em CVPR}, pages 15750--15758, 2021.

\bibitem{chen_iterative_2021}
Yining Chen, Elan Rosenfeld, Mark Sellke, Tengyu Ma, and Andrej Risteski.
\newblock Iterative feature matching: {Toward} provable domain generalization
  with logarithmic environments.
\newblock In {\em NeurIPS}, 2022.

\bibitem{deng2009imagenet}
Jia Deng, Wei Dong, Richard Socher, Li-Jia Li, Kai Li, and Li Fei-Fei.
\newblock Imagenet: A large-scale hierarchical image database.
\newblock In {\em CVPR}, pages 248--255, 2009.

\bibitem{lowrank}
Zhengming Ding and Yun Fu.
\newblock Deep domain generalization with structured low-rank constraint.
\newblock {\em IEEE TIP}, pages 304--313, 2018.

\bibitem{Dubey_2021_CVPR_adapt}
Abhimanyu Dubey, Vignesh Ramanathan, Alex Pentland, and Dhruv Mahajan.
\newblock Adaptive methods for real-world domain generalization.
\newblock In {\em CVPR}, pages 14340--14349, 2021.

\bibitem{Fang_2013_ICCV_vlcs}
Chen Fang, Ye Xu, and Daniel~N. Rockmore.
\newblock Unbiased metric learning: On the utilization of multiple datasets and
  web images for softening bias.
\newblock In {\em ICCV}, 2013.

\bibitem{JMLR:v17:15-239_dann}
Yaroslav Ganin, Evgeniya Ustinova, Hana Ajakan, Pascal Germain, Hugo
  Larochelle, Fran{\c{c}}ois Laviolette, Mario March, and Victor Lempitsky.
\newblock Domain-adversarial training of neural networks.
\newblock {\em Journal of Machine Learning Research}, 17(59):1--35, 2016.

\bibitem{NEURIPS2020_f3ada80d_BYOL}
Jean-Bastien Grill, Florian Strub, Florent Altch\'{e}, Corentin Tallec, Pierre
  Richemond, Elena Buchatskaya, Carl Doersch, Bernardo Avila~Pires, Zhaohan
  Guo, Mohammad Gheshlaghi~Azar, Bilal Piot, koray kavukcuoglu, Remi Munos, and
  Michal Valko.
\newblock Bootstrap your own latent - a new approach to self-supervised
  learning.
\newblock In {\em NeurIPS}, pages 21271--21284, 2020.

\bibitem{gulrajani2021in}
Ishaan Gulrajani and David Lopez-Paz.
\newblock In search of lost domain generalization.
\newblock In {\em ICLR}, 2021.

\bibitem{NEURIPS2021_27debb43_graph}
Jeff~Z. HaoChen, Colin Wei, Adrien Gaidon, and Tengyu Ma.
\newblock Provable guarantees for self-supervised deep learning with spectral
  contrastive loss.
\newblock In {\em NeurIPS}, pages 5000--5011, 2021.

\bibitem{He_2020_CVPR}
Kaiming He, Haoqi Fan, Yuxin Wu, Saining Xie, and Ross Girshick.
\newblock Momentum contrast for unsupervised visual representation learning.
\newblock In {\em CVPR}, 2020.

\bibitem{He_2016_CVPR_resnet}
Kaiming He, Xiangyu Zhang, Shaoqing Ren, and Jian Sun.
\newblock Deep residual learning for image recognition.
\newblock In {\em CVPR}, 2016.

\bibitem{Hendrycks_2021_ICCV}
Dan Hendrycks, Steven Basart, Norman Mu, Saurav Kadavath, Frank Wang, Evan
  Dorundo, Rahul Desai, Tyler Zhu, Samyak Parajuli, Mike Guo, Dawn Song, Jacob
  Steinhardt, and Justin Gilmer.
\newblock The many faces of robustness: A critical analysis of
  out-of-distribution generalization.
\newblock In {\em ICCV}, pages 8340--8349, 2021.

\bibitem{NEURIPS2019_selfrobust}
Dan Hendrycks, Mantas Mazeika, Saurav Kadavath, and Dawn Song.
\newblock Using self-supervised learning can improve model robustness and
  uncertainty.
\newblock In {\em NeurIPS}, 2019.

\bibitem{hochreiter1997flat}
Sepp Hochreiter and J{\"u}rgen Schmidhuber.
\newblock Flat minima.
\newblock {\em Neural computation}, 9(1):1--42, 1997.

\bibitem{rsc}
Zeyi Huang, Haohan Wang, Eric~P. Xing, and Dong Huang.
\newblock Self-challenging improves cross-domain generalization.
\newblock In {\em ECCV}, pages 124--140, 2020.

\bibitem{NEURIPS2021_testtime_adjust}
Yusuke Iwasawa and Yutaka Matsuo.
\newblock Test-time classifier adjustment module for model-agnostic domain
  generalization.
\newblock In {\em NeurIPS}, pages 2427--2440, 2021.

\bibitem{keskar2017flat2}
Nitish~Shirish Keskar, Dheevatsa Mudigere, Jorge Nocedal, Mikhail Smelyanskiy,
  and Ping Tak~Peter Tang.
\newblock On large-batch training for deep learning: Generalization gap and
  sharp minima.
\newblock In {\em ICLR}, 2017.

\bibitem{Kim_2021_ICCV_selfreg}
Daehee Kim, Youngjun Yoo, Seunghyun Park, Jinkyu Kim, and Jaekoo Lee.
\newblock Selfreg: Self-supervised contrastive regularization for domain
  generalization.
\newblock In {\em ICCV}, pages 9619--9628, 2021.

\bibitem{DBLP:journals/corr/KingmaB14_adam}
Diederik~P. Kingma and Jimmy Ba.
\newblock Adam: A method for stochastic optimization.
\newblock In {\em ICLR}, 2015.

\bibitem{pmlr-v139-krueger21a_vrex}
David Krueger, Ethan Caballero, Joern-Henrik Jacobsen, Amy Zhang, Jonathan
  Binas, Dinghuai Zhang, Remi~Le Priol, and Aaron Courville.
\newblock Out-of-distribution generalization via risk extrapolation (rex).
\newblock In {\em ICML}, pages 5815--5826, 2021.

\bibitem{kumar2022finetuning}
Ananya Kumar, Aditi Raghunathan, Robbie~Matthew Jones, Tengyu Ma, and Percy
  Liang.
\newblock Fine-tuning can distort pretrained features and underperform
  out-of-distribution.
\newblock In {\em ICLR}, 2022.

\bibitem{Li_Yang_Song_Hospedales_2018_mldg}
Da Li, Yongxin Yang, Yi-Zhe Song, and Timothy Hospedales.
\newblock Learning to generalize: Meta-learning for domain generalization.
\newblock In {\em AAAI}, 2018.

\bibitem{Li_2017_ICCV_pacs}
Da Li, Yongxin Yang, Yi-Zhe Song, and Timothy~M. Hospedales.
\newblock Deeper, broader and artier domain generalization.
\newblock In {\em ICCV}, 2017.

\bibitem{Li_2018_CVPR_mmdaae}
Haoliang Li, Sinno~Jialin Pan, Shiqi Wang, and Alex~C. Kot.
\newblock Domain generalization with adversarial feature learning.
\newblock In {\em CVPR}, pages 5400--5409, 2018.

\bibitem{li2018domain}
Ya Li, Mingming Gong, Xinmei Tian, Tongliang Liu, and Dacheng Tao.
\newblock Domain generalization via conditional invariant representations.
\newblock In {\em AAAI}, volume~32, 2018.

\bibitem{li2018ciddg}
Ya Li, Xinmei Tian, Mingming Gong, Yajing Liu, Tongliang Liu, Kun Zhang, and
  Dacheng Tao.
\newblock Deep domain generalization via conditional invariant adversarial
  networks.
\newblock In {\em ECCV}, pages 624--639, 2018.

\bibitem{Lv_2022_CVPR_cirl}
Fangrui Lv, Jian Liang, Shuang Li, Bin Zang, Chi~Harold Liu, Ziteng Wang, and
  Di Liu.
\newblock Causality inspired representation learning for domain generalization.
\newblock In {\em CVPR}, pages 8046--8056, 2022.

\bibitem{pmlr-v139-mahajan21b}
Divyat Mahajan, Shruti Tople, and Amit Sharma.
\newblock Domain generalization using causal matching.
\newblock In {\em ICML}, pages 7313--7324, 2021.

\bibitem{Mansilla_2021_ICCV_gradient}
Lucas Mansilla, Rodrigo Echeveste, Diego~H. Milone, and Enzo Ferrante.
\newblock Domain generalization via gradient surgery.
\newblock In {\em ICCV}, pages 6630--6638, 2021.

\bibitem{mohri2018foundations}
Mehryar Mohri, Afshin Rostamizadeh, and Ameet Talwalkar.
\newblock {\em Foundations of machine learning}.
\newblock MIT press, 2018.

\bibitem{pmlr-v28-muandet13}
Krikamol Muandet, David Balduzzi, and Bernhard Schölkopf.
\newblock Domain generalization via invariant feature representation.
\newblock In {\em ICML}, pages 10--18, 2013.

\bibitem{Nam_2021_CVPR_sagnet}
Hyeonseob Nam, HyunJae Lee, Jongchan Park, Wonjun Yoon, and Donggeun Yoo.
\newblock Reducing domain gap by reducing style bias.
\newblock In {\em CVPR}, pages 8690--8699, 2021.

\bibitem{Peng_2019_ICCV_domainnet}
Xingchao Peng, Qinxun Bai, Xide Xia, Zijun Huang, Kate Saenko, and Bo Wang.
\newblock Moment matching for multi-source domain adaptation.
\newblock In {\em ICCV}, 2019.

\bibitem{pmlr-v119-raghunathan20a_tradeoff1}
Aditi Raghunathan, Sang~Michael Xie, Fanny Yang, John Duchi, and Percy Liang.
\newblock Understanding and mitigating the tradeoff between robustness and
  accuracy.
\newblock In {\em ICML}, pages 7909--7919, 2020.

\bibitem{rosenfeld_risks_2021}
Elan Rosenfeld, Pradeep Ravikumar, and Andrej Risteski.
\newblock The risks of invariant risk minimization.
\newblock In {\em {ICLR}}, 2021.

\bibitem{ILSVRC15}
Olga Russakovsky, Jia Deng, Hao Su, Jonathan Krause, Sanjeev Satheesh, Sean Ma,
  Zhiheng Huang, Andrej Karpathy, Aditya Khosla, Michael Bernstein,
  Alexander~C. Berg, and Li Fei-Fei.
\newblock {ImageNet Large Scale Visual Recognition Challenge}.
\newblock {\em IJCV}, 115(3):211--252, 2015.

\bibitem{Sagawa*2020Distributionally_dro}
Shiori Sagawa*, Pang~Wei Koh*, Tatsunori~B. Hashimoto, and Percy Liang.
\newblock Distributionally robust neural networks.
\newblock In {\em ICLR}, 2020.

\bibitem{sun2016deep}
Baochen Sun and Kate Saenko.
\newblock Deep coral: Correlation alignment for deep domain adaptation.
\newblock In {\em ECCV}, pages 443--450, 2016.

\bibitem{vapnik1999nature}
Vladimir Vapnik.
\newblock {\em The nature of statistical learning theory}.
\newblock Springer science \& business media, 1999.

\bibitem{Venkateswara_2017_officehome}
Hemanth Venkateswara, Jose Eusebio, Shayok Chakraborty, and Sethuraman
  Panchanathan.
\newblock Deep hashing network for unsupervised domain adaptation.
\newblock In {\em CVPR}, 2017.

\bibitem{wang_provable_2022}
Haoxiang Wang, Haozhe Si, Bo Li, and Han Zhao.
\newblock Provable domain generalization via invariant-feature subspace
  recovery.
\newblock In {\em {ICML}}, pages 23018--23033, 2022.

\bibitem{Wang_2022_CVPR_cadapt}
Qin Wang, Olga Fink, Luc Van~Gool, and Dengxin Dai.
\newblock Continual test-time domain adaptation.
\newblock In {\em CVPR}, pages 7201--7211, 2022.

\bibitem{pmlr-v119-wang20k_alignment_uniformity}
Tongzhou Wang and Phillip Isola.
\newblock Understanding contrastive representation learning through alignment
  and uniformity on the hypersphere.
\newblock In {\em ICML}, pages 9929--9939, 2020.

\bibitem{domainmixup}
Yufei Wang, Haoliang Li, and Alex~C. Kot.
\newblock Heterogeneous domain generalization via domain mixup.
\newblock In {\em ICASSP}, pages 3622--3626, 2020.

\bibitem{wang2021revisit}
Yulin Wang, Zanlin Ni, Shiji Song, Le Yang, and Gao Huang.
\newblock Revisiting locally supervised learning: an alternative to end-to-end
  training.
\newblock In {\em ICLR}, 2021.

\bibitem{Wu_2021_ICCV}
Aming Wu, Rui Liu, Yahong Han, Linchao Zhu, and Yi Yang.
\newblock Vector-decomposed disentanglement for domain-invariant object
  detection.
\newblock In {\em ICCV}, pages 9342--9351, 2021.

\bibitem{Wu_2018_CVPR}
Zhirong Wu, Yuanjun Xiong, Stella~X. Yu, and Dahua Lin.
\newblock Unsupervised feature learning via non-parametric instance
  discrimination.
\newblock In {\em CVPR}, 2018.

\bibitem{xie2021innout_tradeoff2}
Sang~Michael Xie, Ananya Kumar, Robbie Jones, Fereshte Khani, Tengyu Ma, and
  Percy Liang.
\newblock In-n-out: Pre-training and self-training using auxiliary information
  for out-of-distribution robustness.
\newblock In {\em ICLR}, 2021.

\bibitem{Xu_2021_CVPR_fourier}
Qinwei Xu, Ruipeng Zhang, Ya Zhang, Yanfeng Wang, and Qi Tian.
\newblock A fourier-based framework for domain generalization.
\newblock In {\em CVPR}, pages 14383--14392, 2021.

\bibitem{yan2020improve_mixup_domainbed}
Shen Yan, Huan Song, Nanxiang Li, Lincan Zou, and Liu Ren.
\newblock Improve unsupervised domain adaptation with mixup training.
\newblock {\em arXiv preprint arXiv:2001.00677}, 2020.

\bibitem{Yao_2022_CVPR_pcl}
Xufeng Yao, Yang Bai, Xinyun Zhang, Yuechen Zhang, Qi Sun, Ran Chen, Ruiyu Li,
  and Bei Yu.
\newblock Pcl: Proxy-based contrastive learning for domain generalization.
\newblock In {\em CVPR}, pages 7097--7107, 2022.

\bibitem{Ye_2022_CVPR_oodbench}
Nanyang Ye, Kaican Li, Haoyue Bai, Runpeng Yu, Lanqing Hong, Fengwei Zhou,
  Zhenguo Li, and Jun Zhu.
\newblock Ood-bench: Quantifying and understanding two dimensions of
  out-of-distribution generalization.
\newblock In {\em CVPR}, pages 7947--7958, 2022.

\bibitem{pmlr-v139-barlow}
Jure Zbontar, Li Jing, Ishan Misra, Yann LeCun, and Stephane Deny.
\newblock Barlow twins: Self-supervised learning via redundancy reduction.
\newblock In {\em ICML}, pages 12310--12320, 2021.

\bibitem{zhang2018mixup}
Hongyi Zhang, Moustapha Cisse, Yann~N. Dauphin, and David Lopez-Paz.
\newblock mixup: Beyond empirical risk minimization.
\newblock In {\em ICLR}, 2018.

\bibitem{NEURIPS2021_c705112d_arm}
Marvin Zhang, Henrik Marklund, Nikita Dhawan, Abhishek Gupta, Sergey Levine,
  and Chelsea Finn.
\newblock Adaptive risk minimization: Learning to adapt to domain shift.
\newblock In {\em NeurIPS}, pages 23664--23678, 2021.

\bibitem{zhou2020ddaig}
Kaiyang Zhou, Yongxin Yang, Timothy Hospedales, and Tao Xiang.
\newblock Deep domain-adversarial image generation for domain generalisation.
\newblock In {\em AAAI}, pages 13025--13032, 2020.

\bibitem{zhou2020learning_L2A-OT}
Kaiyang Zhou, Yongxin Yang, Timothy Hospedales, and Tao Xiang.
\newblock Learning to generate novel domains for domain generalization.
\newblock In {\em ECCV}, pages 561--578, 2020.

\bibitem{9540778dael}
Kaiyang Zhou, Yongxin Yang, Yu Qiao, and Tao Xiang.
\newblock Domain adaptive ensemble learning.
\newblock {\em IEEE TIP}, 30:8008--8018, 2021.

\bibitem{zhou2021mixstyle}
Kaiyang Zhou, Yongxin Yang, Yu Qiao, and Tao Xiang.
\newblock Domain generalization with mixstyle.
\newblock In {\em ICLR}, 2021.

\end{thebibliography}
}

\clearpage
\appendix

\onecolumn


\section{Proofs and Discussion on Theoretical Results}
\label{appendix:proofs}

\subsection{Technical Lemmas}

We first introduce several technical lemmas which the main proof is based on. The first lemma characterizes the form of the Bayes classifier on top of any featurizer $\Phi_{(w_d,w_s)}$\footnote{For brevity, we omit the subscript $(w_d,w_s)$ in $\Phi_{(w_d,w_s)}$ when it is clear from context.} defined by Definition 1.
\begin{lemma}
Under the data generation model described in Section 4, given the featurizer $\Phi_{(w_d, w_s)}$, the classifier
\begin{equation}
\mathrm{sign}\left({\bm{\beta}^*}^\top \begin{bmatrix}\Phi(\rvx) \\ 1\end{bmatrix}\right) = \left\{\begin{array}{ll} 1,&\ \mathrm{if}\ \, {\bm{\beta}^*}^\top \begin{bmatrix}\Phi(\rvx) \\ 1\end{bmatrix} \ge 0 \\ -1, &\ \mathrm{otherwise} \end{array} \right. ,
\label{eq:cls_final}
\end{equation}
where
\begin{equation}
\bm{\beta}^* = \begin{bmatrix} \bm{\beta}^*_d\\ \bm{\beta}^*_s\\ \beta^*_0 \end{bmatrix} = \begin{bmatrix} {2 w_d\bm{\mu}_d}/{\sigma_d^2} \\ {2 w_s\bm{\mu}_s}/{\sigma_s^2} \\ \log\frac{\eta}{1-\eta} \end{bmatrix}
\label{eq:bayes}
\end{equation}
is the Bayes classifier in the training domain $\mathcal{D}_\mathrm{train}$.
\label{lemma:bayes}
\end{lemma}
\begin{proof}
Let $(\rvx,\rvy)\sim\mathcal{D}_\mathrm{train}$ and denote by $\widehat{\rvy}$ the model's prediction on $\rvx$. The Bayes classifier is then given by
\begin{equation}
\widehat{\rvy} = \left\{\begin{array}{ll} 1,& \text{if}\ \; p(\rvy=1\,|\,\widetilde{\rvz}_d,\widetilde{\rvz}_s)\ge p(\rvy=-1\,|\,\widetilde{\rvz}_d,\widetilde{\rvz}_s) \\ -1,& \text{otherwise} \end{array} \right. .
\label{eq:bayes_classifier}
\end{equation}
Using Bayes' rule we have
\begin{equation}
p(\rvy=c\,|\,\widetilde{\rvz}_d,\widetilde{\rvz}_s) = \frac{p(\widetilde{\rvz}_d,\widetilde{\rvz}_s\,|\,\rvy=c)p(\rvy=c)}{p(\widetilde{\rvz}_d,\widetilde{\rvz}_s)}\quad \forall c\in\{1,-1\}.
\label{eq:bayes}
\end{equation}
Combining Eqs.~\eqref{eq:bayes_classifier}~\eqref{eq:bayes} and $p(\rvy=1)=\eta$, $p(\rvy=-1)=1-\eta$ gives the Bayes classifier
\begin{equation}
\widehat{\rvy} = \left\{\begin{array}{ll} 1,& \text{if}\ \; p(\widetilde{\rvz}_d,\widetilde{\rvz}_s\,|\,\rvy=1)\cdot\eta \ge p(\widetilde{\rvz}_d,\widetilde{\rvz}_s\,|\,\rvy=-1)\cdot(1-\eta) \\ -1,& \text{otherwise} \end{array} \right.,
\label{eq:bayes_classifier_new}
\end{equation}
where
\begin{equation}
\begin{aligned}
p(\widetilde{\rvz}_d,\widetilde{\rvz}_s\,|\,\rvy=1) &= p(\widetilde{\rvz}_d\,|\,\rvy=1) p(\widetilde{\rvz}_s\,|\,\rvy=1) \\
&= \frac{1}{(2\pi)^{\frac{p_d+p_s}{2}}\sigma_d^{p_d}\sigma_s^{p_s}}\exp\left[-\frac{(\widetilde{\rvz}_d - w_d{\bm{\mu}}_d)^\top(\widetilde{\rvz}_d - w_d{\bm{\mu}}_d)}{2\sigma_d^2} - \frac{(\widetilde{\rvz}_s - w_s{\bm{\mu}}_s)^\top(\widetilde{\rvz}_s - w_s{\bm{\mu}}_s)}{2\sigma_s^2}\right]
\end{aligned}
\end{equation}
and
\begin{equation}
\begin{aligned}
p(\widetilde{\rvz}_d,\widetilde{\rvz}_s\,|\,\rvy=-1) &= p(\widetilde{\rvz}_d\,|\,\rvy=-1) p(\widetilde{\rvz}_s\,|\,\rvy=-1) \\
&= \frac{1}{(2\pi)^{\frac{p_d+p_s}{2}}\sigma_d^{p_d}\sigma_s^{p_s}}\exp\left[-\frac{(\widetilde{\rvz}_d + w_d{\bm{\mu}}_d)^\top(\widetilde{\rvz}_d + w_d{\bm{\mu}}_d)}{2\sigma_d^2} - \frac{(\widetilde{\rvz}_s + w_s{\bm{\mu}}_s)^\top(\widetilde{\rvz}_s + w_s{\bm{\mu}}_s)}{2\sigma_s^2}\right].
\end{aligned}
\end{equation}
Therefore,
\begin{align}
&p(\widetilde{\rvz}_d,\widetilde{\rvz}_s\,|\,\rvy=1)\cdot\eta \ge p(\widetilde{\rvz}_d,\widetilde{\rvz}_s\,|\,\rvy=-1)\cdot (1-\eta)\\
\Longleftrightarrow&\frac{\eta}{1-\eta}\cdot \exp\left[-\frac{(\widetilde{\rvz}_d - w_d{\bm{\mu}}_d)^\top(\widetilde{\rvz}_d - w_d{\bm{\mu}}_d)}{2\sigma_d^2} - \frac{(\widetilde{\rvz}_s - w_s{\bm{\mu}}_s)^\top(\widetilde{\rvz}_s - w_s{\bm{\mu}}_s)}{2\sigma_s^2}\right] \\ &\qquad \ge \exp\left[-\frac{(\widetilde{\rvz}_d + w_d{\bm{\mu}}_d)^\top(\widetilde{\rvz}_d + w_d{\bm{\mu}}_d)}{2\sigma_d^2} - \frac{(\widetilde{\rvz}_s + w_s{\bm{\mu}}_s)^\top(\widetilde{\rvz}_s + w_s{\bm{\mu}}_s)}{2\sigma_s^2}\right] \\
\Longleftrightarrow& -\frac{(\widetilde{\rvz}_d - w_d{\bm{\mu}}_d)^\top(\widetilde{\rvz}_d - w_d{\bm{\mu}}_d)}{2\sigma_d^2} - \frac{(\widetilde{\rvz}_s - w_s{\bm{\mu}}_s)^\top(\widetilde{\rvz}_s - w_s{\bm{\mu}}_s)}{2\sigma_s^2} + \log\frac{\eta}{1-\eta} \\
&\qquad \ge -\frac{(\widetilde{\rvz}_d + w_d{\bm{\mu}}_d)^\top(\widetilde{\rvz}_d + w_d{\bm{\mu}}_d)}{2\sigma_d^2} - \frac{(\widetilde{\rvz}_s + w_s{\bm{\mu}}_s)^\top(\widetilde{\rvz}_s + w_s{\bm{\mu}}_s)}{2\sigma_s^2}\\
\Longleftrightarrow& \frac{2w_d\bm{\mu}_d^\top\widetilde{\rvz}_d}{\sigma_d^2} + \frac{2w_s\bm{\mu}_s^\top\widetilde{\rvz}_s}{\sigma_s^2} + \log\frac{\eta}{1-\eta} \ge 0.\label{eq:cls}
\end{align}
Defining $\bm{\beta}^* = \begin{bmatrix} \bm{\beta}^*_d\\ \bm{\beta}^*_s\\ \beta^*_0 \end{bmatrix} = \begin{bmatrix} {2 w_d\bm{\mu}_d}/{\sigma_d^2} \\ {2 w_s\bm{\mu}_s}/{\sigma_s^2} \\ \log\frac{\eta}{1-\eta} \end{bmatrix}$, Eq.~\eqref{eq:cls} amounts to ${\bm{\beta}^*}^\top\begin{bmatrix} \widetilde{\rvz}_d\\ \widetilde{\rvz}_s \\ 1 \end{bmatrix} = {\bm{\beta}^*}^\top \begin{bmatrix}\Phi(\rvx) \\ 1\end{bmatrix} \ge 0$. Thus, the Bayes classifier~\eqref{eq:bayes_classifier_new} is equivalent to the classifier defined by Eq.~\eqref{eq:cls_final}.
\end{proof}

In other words, the Bayes classifier of our data generation model is contained in the space containing all \emph{linear} classifiers on top of Gaussian features. In what follows, for any linear classifier $\bm{\beta} = ({\bm{\beta}_d, \bm{\beta}_s, \beta_0})^\top$, we denote the binary predictor $\mathrm{sign}\left({\bm{\beta}}^\top \begin{bmatrix}\Phi(\rvx) \\ 1\end{bmatrix}\right)$ by $\bm{\beta}\circ\Phi$. Our next lemma gives the expected train and test risks of this predictor in our DG setting.

\begin{lemma}
Given the featurizer $\Phi_{(w_d, w_s)}$ and the classifier $\bm{\beta} = ({\bm{\beta}_d, \bm{\beta}_s, \beta_0})^\top$, the predictor $\bm{\beta}\circ\Phi$ yields the following expected training and test risks:
\begin{equation}
\begin{aligned}
R_\mathrm{train}(\bm{\beta}\circ\Phi) &= \eta\cdot F\left(-\frac{w_d\bm{\beta}_d^\top \bm{\mu}_d + w_s \bm{\beta}_s^\top \bm{\mu}_s + \beta_0}{\sqrt{\sigma_d^2\lVert \bm{\beta}_d \rVert_2^2 + \sigma_s^2 \lVert \bm{\beta}_s \rVert_2^2}}\right) + (1-\eta)\cdot F\left(-\frac{w_d\bm{\beta}_d^\top \bm{\mu}_d + w_s\bm{\beta}_s^\top \bm{\mu}_s - \beta_0}{\sqrt{\sigma_d^2\lVert \bm{\beta}_d \rVert_2^2 + \sigma_s^2 \lVert \bm{\beta}_s \rVert_2^2}}\right), \\
R_\mathrm{test}(\bm{\beta}\circ\Phi) &= \eta\cdot F\left(-\frac{w_d\bm{\beta}_d^\top \bm{\mu}_d + w_s \gamma\bm{\beta}_s^\top \bm{\mu}_s + \beta_0}{\sqrt{\sigma_d^2\lVert \bm{\beta}_d \rVert_2^2 + \sigma_s^2 \lVert \bm{\beta}_s \rVert_2^2}}\right) + (1-\eta)\cdot F\left(-\frac{w_d\bm{\beta}_d^\top \bm{\mu}_d + w_s\gamma\bm{\beta}_s^\top \bm{\mu}_s - \beta_0}{\sqrt{\sigma_d^2\lVert \bm{\beta}_d \rVert_2^2 + \sigma_s^2 \lVert \bm{\beta}_s \rVert_2^2}}\right),
\end{aligned}
\end{equation}
where $F$ is the CDF of a standard Gaussian $\mathcal{N}(0,1)$.
\label{lemma:error}
\end{lemma}
\begin{proof}
In the training domain, we have
\begin{equation}
    \widetilde{\rvz}_d\sim\mathcal{N}(\rvy\cdot w_d\bm{\mu}_d,\sigma_d^2\bm{I}),\quad \widetilde{\rvz}_s\sim\mathcal{N}(\rvy\cdot w_s\bm{\mu}_s,\sigma^2_s\bm{I})
\end{equation}
Consider the examples with $\rvy=1$. The expected training risk is given by the probability of the quantity $\bm{\beta}_d^\top\widetilde{\rvz}_d + \bm{\beta}_s^\top\widetilde{\rvz}_s + \beta_0$ being negative. Note that 
\begin{equation}
w_d\bm{\beta}_d^\top\bm{\mu}_d + w_s\bm{\beta}_s^\top\bm{\mu}_s + \beta_0 \sim \mathcal{N}\left(w_d\bm{\beta}_d^\top\bm{\mu}_d + w_s\bm{\beta}_s^\top\bm{\mu}_s + \beta_0,\sigma_d^2\lVert\bm{\beta}_d\rVert_2^2 + \sigma_s^2\lVert\bm{\beta}_s\rVert_2^2\right)
\end{equation}
due to the property of multi-variate Gaussian, thus this probability is given by $F\left(-\frac{w_d\bm{\beta}_d^\top\bm{\mu}_d + w_s\bm{\beta}_s^\top\bm{\mu}_s + \beta_0}{\sqrt{\sigma_d^2\lVert\bm{\beta}_d\rVert_2^2 + \sigma_s^2\lVert\bm{\beta}_s\rVert_2^2}} \right)$.
Similarly, for the examples with $\rvy=-1$, the expected training risk is given by the probability of the quantity $\bm{\beta}_d^\top\widetilde{\rvz}_d + \bm{\beta}_s^\top\widetilde{\rvz}_s + \beta_0$ being positive. Note that
\begin{equation}
    w_d\bm{\beta}_d^\top\bm{\mu}_d + w_s\bm{\beta}_s^\top\bm{\mu}_s + \beta_0 \sim \mathcal{N}\left(-w_d\bm{\beta}_d^\top\bm{\mu}_d - w_s\bm{\beta}_s^\top\bm{\mu}_s + \beta_0,\sigma_d^2\lVert\bm{\beta}_d\rVert_2^2 + \sigma_s^2\lVert\bm{\beta}_s\rVert_2^2\right),
\end{equation}
thus this probability is given by $F\left(-\frac{w_d\bm{\beta}_d^\top\bm{\mu}_d + w_s\bm{\beta}_s^\top\bm{\mu}_s - \beta_0}{\sqrt{\sigma_d^2\lVert\bm{\beta}_d\rVert_2^2 + \sigma_s^2\lVert\bm{\beta}_s\rVert_2^2}} \right)$. Therefore, the expected training risk of the predictor $\bm{\beta}\circ\Phi$ is
\begin{equation}
    R_\mathrm{train}(\bm{\beta}\circ\Phi) = \eta\cdot F\left(-\frac{w_d\bm{\beta}_d^\top \bm{\mu}_d + w_s \bm{\beta}_s^\top \bm{\mu}_s + \beta_0}{\sqrt{\sigma_d^2\lVert \bm{\beta}_d \rVert_2^2 + \sigma_s^2 \lVert \bm{\beta}_s \rVert_2^2}}\right) + (1-\eta)\cdot F\left(-\frac{w_d\bm{\beta}_d^\top \bm{\mu}_d + w_s\bm{\beta}_s^\top \bm{\mu}_s - \beta_0}{\sqrt{\sigma_d^2\lVert \bm{\beta}_d \rVert_2^2 + \sigma_s^2 \lVert \bm{\beta}_s \rVert_2^2}}\right).
\end{equation}
The expected test risk of the predictor can be derived analogously using the fact that
\begin{equation}
    \widetilde{\rvz}_d\sim\mathcal{N}(\rvy\cdot w_d\bm{\mu}_d,\sigma_d^2\bm{I}),\quad \widetilde{\rvz}_s\sim\mathcal{N}(\rvy\cdot w_s\gamma\bm{\mu}_s,\sigma^2_s\bm{I})
\end{equation}
in the test domain.
\end{proof}

\subsection{Proof of Theorem~1}

We first state a generalized version of Theorem~1 in the main text and give its proof.

\begin{theorem}[Expected test domain risk, generalized]
For any $w_d,w_s\in[0,1]$, the predictor $g^*\circ\Phi_{(w_d,w_s)}$, composed of the featurizer $\Phi_{(w_d,w_s)}$ and the training-domain Bayes classifier $g^*$ on top of $\Phi_{(w_d,w_s)}$, gives the expected train and test risks of
\begin{equation}
\begin{aligned}
R_\mathrm{train}(g^*\circ\Phi) &= \eta\cdot F\left(-\frac{\frac{w_d^2\lVert \bm{\mu}_d \rVert_2^2}{\sigma_d^2} + \frac{w_s^2\lVert\bm{\mu}_s\rVert_2^2}{\sigma_s^2} + \log \frac{\eta}{1-\eta}}{\sqrt{\frac{w_d^2\lVert \bm{\mu}_d \rVert_2^2}{\sigma_d^2} + \frac{w_s^2\lVert\bm{\mu}_s\rVert_2^2}{\sigma_s^2}}}\right) + (1-\eta)\cdot F\left(-\frac{\frac{w_d^2\lVert \bm{\mu}_d \rVert_2^2}{\sigma_d^2} + \frac{w_s^2\lVert\bm{\mu}_s\rVert_2^2}{\sigma_s^2} - \log \frac{\eta}{1-\eta}}{\sqrt{\frac{w_d^2\lVert \bm{\mu}_d \rVert_2^2}{\sigma_d^2} + \frac{w_s^2\lVert\bm{\mu}_s\rVert_2^2}{\sigma_s^2}}}\right), \\
R_\mathrm{test}(g^*\circ \Phi) &= \eta\cdot F\left(-\frac{\frac{w_d^2\lVert \bm{\mu}_d \rVert_2^2}{\sigma_d^2} + \frac{w_s^2\gamma\lVert\bm{\mu}_s\rVert_2^2}{\sigma_s^2} + \log \frac{\eta}{1-\eta}}{\sqrt{\frac{w_d^2\lVert \bm{\mu}_d \rVert_2^2}{\sigma_d^2} + \frac{w_s^2\lVert\bm{\mu}_s\rVert_2^2}{\sigma_s^2}}}\right) + (1-\eta)\cdot F\left(-\frac{\frac{w_d^2\lVert \bm{\mu}_d \rVert_2^2}{\sigma_d^2} + \frac{w_s^2\gamma\lVert\bm{\mu}_s\rVert_2^2}{\sigma_s^2} - \log \frac{\eta}{1-\eta}}{\sqrt{\frac{w_d^2\lVert \bm{\mu}_d \rVert_2^2}{\sigma_d^2} + \frac{w_s^2\lVert\bm{\mu}_s\rVert_2^2}{\sigma_s^2}}}\right),
\end{aligned}
\label{eq:errors}
\end{equation}
where $F$ is the CDF of a standard Gaussian $\mathcal{N}(0,1)$.
\label{theo:errors}
\end{theorem}

\begin{proof}
According to Lemma~\ref{lemma:bayes}, on top of the featurizer $\Phi_{(w_d,w_s)}$, the training-domain Bayes classifier $g^*$ is given by the linear classifier $\bm{\beta}^* = \begin{bmatrix} \bm{\beta}^*_d\\ \bm{\beta}^*_s\\ \beta^*_0 \end{bmatrix} = \begin{bmatrix} {2 w_d\bm{\mu}_d}/{\sigma_d^2} \\ {2 w_s\bm{\mu}_s}/{\sigma_s^2} \\ \log\frac{\eta}{1-\eta} \end{bmatrix}$. Replacing all $\bm{\beta}$ by $\bm{\beta}^*$ in Lemma~\ref{lemma:error} gives the desired result.
\end{proof}

\paragraph{Recovering Theorem 1 in the main text:} Incorporating the additional assumptions of $\eta=\frac{1}{2}$ and $\sigma^2_d = \sigma^2_s = \sigma^2$ into $R_\mathrm{train}(g^*\circ\Phi)$ and $R_\mathrm{test}(g^*\circ\Phi)$ in Eq.~\eqref{eq:errors} gives
\begin{equation}
R_\mathrm{train}\left(g^*\circ\Phi\right) = F\left(-\frac{1}{\sigma}\cdot\frac{{w}_d^2\lVert\bm{\mu}_d\rVert_2^2 + {w}_s^2\lVert\bm{\mu}_s\rVert_2^2}{\sqrt{ {w}_d^2\lVert\bm{\mu}_d\rVert_2^2 +{w}_s^2 \lVert \bm{\mu}_s \rVert_2^2}}\right)
\label{eq:train}
\end{equation}
and
\begin{equation}
R_\mathrm{test}\left(g^*\circ\Phi\right) = F\left(-\frac{1}{\sigma}\cdot\frac{{w}_d^2\lVert\bm{\mu}_d\rVert_2^2 + \gamma{w}_s^2\lVert\bm{\mu}_s\rVert_2^2}{\sqrt{ {w}_d^2\lVert\bm{\mu}_d\rVert_2^2 +{w}_s^2 \lVert \bm{\mu}_s \rVert_2^2}}\right),
\label{eq:test}
\end{equation}
which is exactly the result in the main text.

\subsection{Discussion on Theorem~1}

In this section, we elaborate on the discussion on the results of Theorem~1. For simplicity, our discussion is based on the result with assumptions $\eta=\frac{1}{2}$ and $\sigma^2_d=\sigma^2_s=\sigma^2$ as in the main text, but a similar discussion also applies to the full theorem without these assumptions in a straightforward manner.

\paragraph{Optimality of invariant representation when $\gamma < 0$.}
As we have briefly mentioned in the main text, Theorem 1 shows that when $\gamma < 0$, the invariant featurizer $\Phi_{(1,0)}$ is indeed optimal. To see this, plugging $w_d=1$ and $w_s=0$ into Eq.~\eqref{eq:test} we observe that
\begin{equation}
    R_\mathrm{test}\left(g^*\circ\Phi_{(1,0)}\right) = F\left( -\frac{\lVert\bm{\mu}_d\rVert_2}{\sigma} \right),
\end{equation}
while for any $w_s>0$ and $\gamma < 0$,
\begin{equation}
\begin{aligned}
    R_\mathrm{test}\left(g^*\circ\Phi_{(1,w_s)}\right) &= F\left( -\frac{1}{\sigma}\cdot \frac{\lVert\bm{\mu}_d\rVert_2^2 + \gamma{w}_s^2\lVert\bm{\mu}_s\rVert_2^2}{\sqrt{ \lVert\bm{\mu}_d\rVert_2^2 +{w}_s^2 \lVert \bm{\mu}_s \rVert_2^2}} \right)\\
    &\ge F\left( -\frac{1}{\sigma}\cdot \frac{\lVert\bm{\mu}_d\rVert_2^2}{\sqrt{ \lVert\bm{\mu}_d\rVert_2^2 +{w}_s^2 \lVert \bm{\mu}_s \rVert_2^2}} \right)\\
    &\ge F\left( -\frac{\lVert\bm{\mu}_d\rVert_2}{\sigma}\right)\\
    &=R_\mathrm{test}\left(g^*\circ\Phi_{(1,0)}\right).
\end{aligned}
\end{equation}

\paragraph{Benefits of leveraging silent feature when $\gamma > 1$.}
When $w_s>0$ and $\gamma > 1$, i.e., the silent feature is \emph{more} predictive in the test domain than in the training domain, we have
\begin{equation}
\begin{aligned}
    R_\mathrm{test}\left(g^*\circ\Phi_{(1,w_s)}\right) &= F\left( -\frac{1}{\sigma}\cdot \frac{\lVert\bm{\mu}_d\rVert_2^2 + \gamma{w}_s^2\lVert\bm{\mu}_s\rVert_2^2}{\sqrt{ \lVert\bm{\mu}_d\rVert_2^2 +{w}_s^2 \lVert \bm{\mu}_s \rVert_2^2}} \right)\\
    &\le F\left( -\frac{1}{\sigma}\cdot \frac{\lVert\bm{\mu}_d\rVert_2^2 + \gamma{w}_s^2\lVert\bm{\mu}_s\rVert_2^2}{\sqrt{ \lVert\bm{\mu}_d\rVert_2^2 +\gamma{w}_s^2 \lVert \bm{\mu}_s \rVert_2^2}} \right)\\
    &= F\left( -\frac{1}{\sigma}\cdot \sqrt{ \lVert\bm{\mu}_d\rVert_2^2 +\gamma{w}_s^2 \lVert \bm{\mu}_s \rVert_2^2} \right)\\
    &\le F\left( -\frac{\lVert\bm{\mu}_d\rVert_2}{\sigma}\right)\\
    &=R_\mathrm{test}\left(g^*\circ\Phi_{(1,0)}\right).
\end{aligned}
\end{equation}
In other words, when the silent feature is more predictive in the test domain, explicitly incorporating silent feature leads to \emph{smaller} test domain risk than the classifier that only use dominant feature.

\section{Experimental Settings}
\label{appendix:experiment_settings}

\subsection{Implementation Details of STEP}

Our proposed Silent Feature Preservation (STEP) consists of three components: pre-trained models with silent features, the LP-FT~\cite{kumar2022finetuning} strategy, and a SWAD~\cite{NEURIPS2021_swad} technique, with full implementation details for each component provided below. First, STEP utilizes the self-supervised contrastive learning pre-trained models. STEP-S use ResNet-50~\cite{He_2016_CVPR_resnet} as its backbone, pre-trained on ImageNet~\cite{deng2009imagenet} for 800 epochs by SwAV~\cite{NEURIPS2020_70feb62b}. It ensures that no more data is introduced compared with the supervised learning pre-training model. For a fair comparison, we use the ResNet-50 architecture (without the batch normalization layer) implemented by DomainBed\cite{gulrajani2021in}. For LP-FT, we divide the original total number of fine-tuning iterations into two parts, linear probing, and fine-tuning, to make sure that the total number of iterations does not exceed the baselines. Finally, we solely employ the SWAD in the fine-tuning process, and we basically use the hyperparameters from the original literature, but with a lower tolerance rate of $r$ = 0.1 and smaller evaluation intervals to better fit the properties of the silent features: 25 for VLCS, 250 for DomainNet, and 50 for the other benchmarks.

\subsection{Hyperparameter Search Protocol}

Considering the computational overhead of the hyperparameter search protocol in DomainBed is too expensive, and STEP introduces three more hyperparameters, both leading to a larger search space, we choose to perform grid search on a discrete parameter space. We select the model with the highest accuracy on the training-domain validation set for each set of hyperparameters. Afterward, unlike DomainBed, we choose the optimal hyperparameters by conducting grid search on the test-domain accuracies of the selected models before. As a result, we use the model selection criterion of the test-domain validation set in DomainBed as baselines for a fair comparison. The hyperparameter search space shared by different methods is shown in~\cref{tab:basic-param}, we fix the relatively insignificant parameters such as batch size, weight decay, and dropout rate, and adopt the same discrete three frequently-used learning rates as SWAD. In addition,~\cref{tab:specific-param-steps} provide the specific hyperparameter search spaces for STEP-S. To further reduce the computational cost, we restrict the maximum linear probe and fine-tune iterations to the corresponding three combinations in the list.

\begin{table*}
  \caption{\textbf{Comparison of shared hyperparameter search space.} U and list denote random search on uniform distribution and grid search on discrete space, respectively.}
  \centering
  \begin{tabular}{llllll}
    \toprule
    \textbf{Parameter} & \textbf{DomainBed} & \textbf{SWAD} & \textbf{Reproduced} & \textbf{STEP-S(Ours)} \\
    \midrule
    batch size & $2^{U(3,5.5)}$ & 32 & 32 & 32\\
    learning rate & $10^{U(-5,-3.5)}$ & [1e-5, 3e-5, 5e-5] & [1e-5, 3e-5, 5e-5] & [1e-4, 2e-4, 3e-4]\\
    weight decay & $10^{U(-6,-2)}$ & [1e-6, 1e-4] & 0.0 & 0.0\\
    dropout rate & [0.0, 0.1, 0.5] & [0.0, 0.1, 0.5] & 0.0 & 0.0\\
    \bottomrule
  \end{tabular}
  \label{tab:basic-param}
\end{table*}

\begin{table*}
  \caption{Comparison of \textbf{specific} hyperparameter search space of \textbf{STEP-S} on each benchmark. List denotes grid search on discrete space. The lr and iters parameters refer to the learning rate and maximum number of iterations, respectively. The abbreviation lp stands for linear probing and ft for fine-tuning.}
  \centering
  \begin{tabular}{llllll}
    \toprule
    \textbf{Param} & VLCS & PACS & Office-Home & TerraInc & DomainNet \\
    \midrule
    lp lr & [1e-3, 5e-3, 1e-2] & [3e-4, 5e-4, 1e-3] & [3e-4, 5e-4, 1e-3] & [3e-4, 5e-4, 1e-3] & [3e-4, 5e-4, 1e-3] \\
    ft lr & [1e-4, 2e-4, 3e-4] & [1e-4, 2e-4, 3e-4] & [1e-4, 2e-4, 3e-4] & [1e-4, 2e-4, 3e-4] & [1e-4, 2e-4, 3e-4] \\
    lp iters & [50, 250, 2500] & [500, 1000, 1500] & [500, 1000, 1500] & [500, 1000, 1500] & [1000, 2000, 3000] \\
    ft iters & [250, 2750, 2500] & [4500, 4000, 3500] & [4500, 4000, 3500] & [4500, 4000, 3500] & [14000, 13000, 12000] \\
    \bottomrule
  \end{tabular}
  \label{tab:specific-param-steps}
\end{table*}

\subsection{Reproducibility}

We provide publicly the source code of STEP, which contains the dataset split files we used, to ensure the reliability and reproducibility of our experimental results. The optimal hyperparameters that correspond to the reported results of STEP-S are also provided in~\cref{tab:optimal-param-steps}, respectively. All experiments are conducted on a single NVIDIA V100 or A100.

\begin{table*}
  \caption{The optimal hyperparameters of \textbf{STEP-S} on each benchmark.}
  \centering
  \begin{tabular}{llllll}
    \toprule
    \textbf{Parameters} & VLCS & PACS & Office-Home & TerraInc & DomainNet \\
    \midrule
    lp lr & 1e-2 & 5e-4 & 3e-4 & 3e-4 & 3e-4 \\
    ft lr & 1e-4 & 3e-4 & 3e-4 & 3e-4 & 3e-4 \\
    lp iters & 50 & 500 & 500 & 500 & 1000 \\
    ft iters & 250 & 4500 & 4500 & 4500 & 14000 \\
    \bottomrule
  \end{tabular}
  \label{tab:optimal-param-steps}
\end{table*}

\section{Supplementary Experiment}
\label{appendix:supp_experiment}

\subsection{Linear Probing Experiment Results}

In order to investigate the generalization superiority of silent features on test distributions with significant distribution shifts, we designed experiments with linear probing, updating the linear head without changing the initial features. The experimental results are reported in~\cref{tab:linear-probing}, and it is evident that the self-supervised pre-trained model SwAV achieves better results on all four benchmarks with larger distribution diversity, whereas supervised learning outperforms SwAV by a large margin on Office-Home with the lowest distribution diversity. Since linear probing does not change the features of the model, this experiment demonstrates the generalization ability of silent features in the case of significant distribution shifts. However, in terms of absolute values of the experimental results, the difference between linear probing only and fine-tuning results is obvious, which is due to the shifts of the downstream DG dataset and ImageNet.

\begin{table*}
  \caption{DG accuracies of different learning frameworks as pre-trained backbones on five DG benchmarks using linear probing only. The \textbf{best results} are marked in bold.}
  \centering
  \begin{tabular}{c|ccccc}
    \toprule
    Methods & VLCS & PACS & Office-Home & TerraInc & DomainNet \\
    \midrule
    Supervised + LP & 76.7 $\pm$ 0.4 & 68.2 $\pm$ 0.3 & \textbf{67.7 $\bm{\pm}$ 0.2} & 37.7 $\pm$ 0.5 & 32.5 $\pm$ 0.0\\
    SwAV + LP & \textbf{77.3 $\bm{\pm}$ 0.0} & \textbf{69.3 $\bm{\pm}$ 0.4} & 63.0 $\pm$ 0.1 & \textbf{38.4 $\bm{\pm}$ 0.4} & \textbf{33.0 $\bm{\pm}$ 0.0}\\
    \bottomrule
  \end{tabular}
  \label{tab:linear-probing}
\end{table*}

\subsection{Additional Empirical Evidence}

We provide more empirical evidence for this in~\cref{tab:empirical_justification}: the average euclidean distance between the STEP-adapted and original features is smaller, and the mutual information with the input is higher, showing that the performance improvement is due to preserving silent features. The experiment of estimating the mutual information based on the reconstruction loss follows the architecture of~\cite{wang2021revisit}.

\begin{table*}
  \caption{Euclidean distance of features before and after training and estimated mutual information between raw inputs and features (smaller reconstruction loss means higher mutual information).}
  \centering
  \begin{tabular}{l|cc}
    \toprule
    Methods & Avg euclidean dist & Avg reconstruction loss \\
    \midrule
    SwAV+ERM & 0.4844 & 0.5074$\pm$0.0003  \\
    SwAV+STEP & \textbf{0.0732} & \textbf{0.5048\bm{$\pm$}0.0000}  \\
    \bottomrule
  \end{tabular}
  \label{tab:empirical_justification}
\end{table*}

\section{Full Results}
\label{appendix:results}

In this section, we show the detailed experimental results for each domain of the five benchmarks of Table 2 in the main text. We highlight the \textbf{best} and \underline{second-best} results in bold and underlined, respectively. Note that ERM (our runs) and SWAD (our runs) are reproduced results under our hyperparameter search protocol, baselines from the reports in DomainBed are marked with $\dagger$, and other results marked with $\ast$ are from the original literature. We selected the results based on the oracle model selection criterion in DomainBed for a fair comparison.

\begin{table*}[htb]
  \caption{Full results of out-of-domain accuracies on VLCS~\cite{Fang_2013_ICCV_vlcs}.}
  \centering
  \begin{tabular}{l|cccc|c}
    \toprule
    Methods & Caltech101 & PASCAL VOC & LabelMe & SUN09 & Avg.\\
    \midrule
    ERM$^{\dagger}$ & 97.6 $\pm$ 0.3 & \underline{67.9 $\pm$ 0.7} & 70.9 $\pm$ 0.2 & 74.0 $\pm$ 0.6 & 77.6 \\
    ERM(reproduced) & 99.1 $\pm$ 0.4 & 63.9 $\pm$ 0.4 & 78.5 $\pm$ 0.4 & 69.9 $\pm$ 0.5 & 77.9 \\
    IRM$^{\dagger}$ & 97.3 $\pm$ 0.2 & 66.7 $\pm$ 0.1 & 71.0 $\pm$ 2.3 & 72.8 $\pm$ 0.4 & 76.9 \\
    GroupDRO$^{\dagger}$ & 97.7 $\pm$ 0.2 & 65.9 $\pm$ 0.2 & 72.8 $\pm$ 0.8 & 73.4 $\pm$ 1.3 & 77.4 \\
    Mixup$^{\dagger}$ & 97.8 $\pm$ 0.4 & 67.2 $\pm$ 0.4 & 71.5 $\pm$ 0.2 & 75.7 $\pm$ 0.6 & 78.1 \\
    MLDG$^{\dagger}$  & 97.1 $\pm$ 0.5 & 66.6 $\pm$ 0.5 & 71.5 $\pm$ 0.1 & 75.0 $\pm$ 0.9 & 77.5 \\
    CORAL$^{\dagger}$ & 97.3 $\pm$ 0.2 & 67.5 $\pm$ 0.6 & 71.6 $\pm$ 0.6 & 74.5 $\pm$ 0.0 & 77.7 \\
    MMD$^{\dagger}$ & 98.8 $\pm$ 0.0       & 66.4 $\pm$ 0.4       & 70.8 $\pm$ 0.5       & 75.6 $\pm$ 0.4       & 77.9 \\
    DANN$^{\dagger}$ & 99.0 $\pm$ 0.2       & 66.3 $\pm$ 1.2       & 73.4 $\pm$ 1.4       & \textbf{80.1 \bm{$\pm$} 0.5}       & \underline{79.7} \\
    CDANN$^{\dagger}$ & 98.2 $\pm$ 0.1       & \textbf{68.8 \bm{$\pm$} 0.5}       & 74.3 $\pm$ 0.6       & 78.1 $\pm$ 0.5       & \textbf{79.9} \\
    MTL$^{\dagger}$ & 97.9 $\pm$ 0.7       & 66.1 $\pm$ 0.7       & 72.0 $\pm$ 0.4       & 74.9 $\pm$ 1.1       & 77.7 \\
    SagNet$^{\dagger}$ & 97.4 $\pm$ 0.3       & 66.4 $\pm$ 0.4       & 71.6 $\pm$ 0.1       & 75.0 $\pm$ 0.8       & 77.6 \\
    ARM$^{\dagger}$ & 97.6 $\pm$ 0.6       & 66.5 $\pm$ 0.3       & 72.7 $\pm$ 0.6       & 74.4 $\pm$ 0.7       & 77.8 \\
    VREx$^{\dagger}$ & 98.4 $\pm$ 0.2       & 66.4 $\pm$ 0.7       & 72.8 $\pm$ 0.1       & 75.0 $\pm$ 1.4       & 78.1 \\
    RSC$^{\dagger}$  & 98.0 $\pm$ 0.4       & 67.2 $\pm$ 0.3       & 70.3 $\pm$ 1.3       & 75.6 $\pm$ 0.4       & 77.8 \\
    SelfReg$^{\ast}$ & 97.4 $\pm$ 0.4 & 63.5 $\pm$ 0.3 & 72.6 $\pm$ 0.1 & 76.7 $\pm$ 0.7 & 77.5 \\
    SWAD$^{\ast}$ & 98.8 $\pm$ 0.1 & 63.3 $\pm$ 0.3 & 75.3 $\pm$ 0.5 & \underline{79.2 $\pm$ 0.6} & 79.1 \\
    SWAD(reproduced) & \underline{99.2 $\pm$ 0.2} & 61.6 $\pm$ 0.8 & \underline{79.4 $\pm$ 0.5} & 72.0 $\pm$ 1.5 & 78.1 \\
    \midrule
    STEP-S(ours) & \textbf{99.4 \bm{$\pm$} 0.1} & 62.9 $\pm$ 1.0 & \textbf{79.5 \bm{$\pm$} 0.2} & 72.7 $\pm$ 0.5 & 78.6 \\
    \bottomrule
  \end{tabular}
  \label{tab:vlcs}
\end{table*}

\begin{table*}[htb]
  \caption{Full results of out-of-domain accuracies on PACS~\cite{Li_2017_ICCV_pacs}.}
  \centering
  \begin{tabular}{l|cccc|c}
    \toprule
    Methods & Art Painting & Cartoon & Photo & Sketch & Avg.\\
    \midrule
    ERM$^{\dagger}$ & 86.5 $\pm$ 1.0       & 81.3 $\pm$ 0.6       & 96.2 $\pm$ 0.3       & \underline{82.7 $\pm$ 1.1}       & 86.7 \\
    ERM(reproduced) & 85.2 $\pm$ 1.1 & 81.3 $\pm$ 1.0 & 95.4 $\pm$ 0.9 & 76.6 $\pm$ 1.6 & 84.6 \\
    IRM$^{\dagger}$ & 84.2 $\pm$ 0.9       & 79.7 $\pm$ 1.5       & 95.9 $\pm$ 0.4       & 78.3 $\pm$ 2.1       & 84.5 \\
    GroupDRO$^{\dagger}$ & 87.5 $\pm$ 0.5       & 82.9 $\pm$ 0.6       & 97.1 $\pm$ 0.3       & 81.1 $\pm$ 1.2       & 87.1 \\
    Mixup$^{\dagger}$ & 87.5 $\pm$ 0.4       & 81.6 $\pm$ 0.7       & 97.4 $\pm$ 0.2       & 80.8 $\pm$ 0.9       & 86.8 \\
    MLDG$^{\dagger}$  & 87.0 $\pm$ 1.2       & 82.5 $\pm$ 0.9       & 96.7 $\pm$ 0.3       & 81.2 $\pm$ 0.6       & 86.8 \\
    CORAL$^{\dagger}$ & 86.6 $\pm$ 0.8       & 81.8 $\pm$ 0.9       & 97.1 $\pm$ 0.5       & 82.7 $\pm$ 0.6       & 87.1 \\
    MMD$^{\dagger}$ & 88.1 $\pm$ 0.8       & 82.6 $\pm$ 0.7       & 97.1 $\pm$ 0.5       & 81.2 $\pm$ 1.2       & 87.2 \\
    DANN$^{\dagger}$ & 87.0 $\pm$ 0.4       & 80.3 $\pm$ 0.6       & 96.8 $\pm$ 0.3       & 76.9 $\pm$ 1.1       & 85.2 \\
    CDANN$^{\dagger}$ & 87.7 $\pm$ 0.6       & 80.7 $\pm$ 1.2       & 97.3 $\pm$ 0.4       & 77.6 $\pm$ 1.5       & 85.8 \\
    MTL$^{\dagger}$ & 87.0 $\pm$ 0.2       & 82.7 $\pm$ 0.8       & 96.5 $\pm$ 0.7       & 80.5 $\pm$ 0.8       & 86.7  \\
    SagNet$^{\dagger}$ & 87.4 $\pm$ 0.5       & 81.2 $\pm$ 1.2       & 96.3 $\pm$ 0.8       & 80.7 $\pm$ 1.1       & 86.4 \\
    ARM$^{\dagger}$ & 85.0 $\pm$ 1.2       & 81.4 $\pm$ 0.2       & 95.9 $\pm$ 0.3       & 80.9 $\pm$ 0.5       & 85.8 \\
    VREx$^{\dagger}$ & 87.8 $\pm$ 1.2       & 81.8 $\pm$ 0.7       & 97.4 $\pm$ 0.2       & 82.1 $\pm$ 0.7       & 87.2 \\
    RSC$^{\dagger}$  & 86.0 $\pm$ 0.7       & 81.8 $\pm$ 0.9       & 96.8 $\pm$ 0.7       & 80.4 $\pm$ 0.5       & 86.2 \\
    SelfReg$^{\ast}$ & 85.9 $\pm$ 0.6 & 81.9 $\pm$ 0.4 & 96.8 $\pm$ 0.1 & 81.4 $\pm$ 0.6 & 86.5 \\
    SWAD$^{\ast}$ & \textbf{89.3 \bm{$\pm$} 0.2} & \underline{83.4 $\pm$ 0.6} & 97.3 $\pm$ 0.3 & 82.5 $\pm$ 0.5 & \underline{88.1} \\
    SWAD(reproduced) & 89.0 $\pm$ 0.3 & 82.3 $\pm$ 0.6 & \underline{97.6 $\pm$ 0.1} & 79.5 $\pm$ 1.0 & 87.1 \\
    \midrule
    STEP-S(ours) & \underline{88.8 $\pm$ 0.5} & \textbf{85.6 \bm{$\pm$} 0.8} & \textbf{97.7 \bm{$\pm$} 0.2} & \textbf{82.9 $\pm$ 1.4} & \textbf{88.7} \\
    \bottomrule
  \end{tabular}
  \label{tab:pacs}
\end{table*}

\begin{table*}[htb]
  \caption{Full results of out-of-domain accuracies on Office-Home~\cite{Venkateswara_2017_officehome}.}
  \centering
  \begin{tabular}{l|cccc|c}
    \toprule
    Methods & Art & Clipart & Product & Real-World & Avg.\\
    \midrule
    ERM$^{\dagger}$ & 61.7 $\pm$ 0.7       & 53.4 $\pm$ 0.3       & 74.1 $\pm$ 0.4       & 76.2 $\pm$ 0.6       & 66.4 \\
    ERM(reproduced) & 60.8 $\pm$ 1.2 & 53.8 $\pm$ 0.5 & 75.6 $\pm$ 0.5 & 76.1 $\pm$ 0.8 & 66.6 \\
    IRM$^{\dagger}$ & 56.4 $\pm$ 3.2       & 51.2 $\pm$ 2.3       & 71.7 $\pm$ 2.7       & 72.7 $\pm$ 2.7       & 63.0 \\
    GroupDRO$^{\dagger}$ & 60.5 $\pm$ 1.6       & 53.1 $\pm$ 0.3       & 75.5 $\pm$ 0.3       & 75.9 $\pm$ 0.7       & 66.2 \\
    Mixup$^{\dagger}$ & 63.5 $\pm$ 0.2       & 54.6 $\pm$ 0.4       & 76.0 $\pm$ 0.3       & 78.0 $\pm$ 0.7       & 68.0 \\
    MLDG$^{\dagger}$  & 60.5 $\pm$ 0.7       & 54.2 $\pm$ 0.5       & 75.0 $\pm$ 0.2       & 76.7 $\pm$ 0.5       & 66.6 \\
    CORAL$^{\dagger}$ & 64.8 $\pm$ 0.8       & 54.1 $\pm$ 0.9       & 76.5 $\pm$ 0.4       & 78.2 $\pm$ 0.4       & 68.4 \\
    MMD$^{\dagger}$ & 60.4 $\pm$ 1.0       & 53.4 $\pm$ 0.5       & 74.9 $\pm$ 0.1       & 76.1 $\pm$ 0.7       & 66.2 \\
    DANN$^{\dagger}$ & 60.6 $\pm$ 1.4       & 51.8 $\pm$ 0.7       & 73.4 $\pm$ 0.5       & 75.5 $\pm$ 0.9       & 65.3 \\
    CDANN$^{\dagger}$ & 57.9 $\pm$ 0.2       & 52.1 $\pm$ 1.2       & 74.9 $\pm$ 0.7       & 76.2 $\pm$ 0.2       & 65.3 \\
    MTL$^{\dagger}$ & 60.7 $\pm$ 0.8       & 53.5 $\pm$ 1.3       & 75.2 $\pm$ 0.6       & 76.6 $\pm$ 0.6       & 66.5  \\
    SagNet$^{\dagger}$ & 62.7 $\pm$ 0.5       & 53.6 $\pm$ 0.5       & 76.0 $\pm$ 0.3       & 77.8 $\pm$ 0.1       & 67.5 \\
    ARM$^{\dagger}$ & 58.8 $\pm$ 0.5       & 51.8 $\pm$ 0.7       & 74.0 $\pm$ 0.1       & 74.4 $\pm$ 0.2       & 64.8 \\
    VREx$^{\dagger}$ & 59.6 $\pm$ 1.0       & 53.3 $\pm$ 0.3       & 73.2 $\pm$ 0.5       & 76.6 $\pm$ 0.4       & 65.7 \\
    RSC$^{\dagger}$  & 61.7 $\pm$ 0.8       & 53.0 $\pm$ 0.9       & 74.8 $\pm$ 0.8       & 76.3 $\pm$ 0.5       & 66.5 \\
    SelfReg$^{\ast}$ & \underline{64.9 $\pm$ 0.8} & 55.4 $\pm$ 0.6 & \underline{78.4 $\pm$ 0.2} & 78.8 $\pm$ 0.1 & \underline{69.4} \\
    SWAD$^{\ast}$ & \textbf{66.1 \bm{$\pm$} 0.4} & \textbf{57.7 \bm{$\pm$} 0.4} & \underline{78.4 $\pm$ 0.1} & \textbf{80.2 \bm{$\pm$} 0.5} & \textbf{70.6} \\
    SWAD(reproduced) & \textbf{66.1 \bm{$\pm$} 0.4} & \underline{57.6 $\pm$ 0.8} & \textbf{79.0 \bm{$\pm$} 0.1} & 79.5 $\pm$ 0.4 & \textbf{70.6} \\
    \midrule
    STEP-S(ours) & 64.4 $\pm$ 0.4 & 52.4 $\pm$ 0.4 & 78.2 $\pm$ 0.2 & \underline{79.6 $\pm$ 0.1} & 68.6 \\
    \bottomrule
  \end{tabular}
  \label{tab:office}
\end{table*}

\begin{table*}[htb]
  \caption{Full results of out-of-domain accuracies on Terra Incognita~\cite{Beery_2018_ECCV_terra}.}
  \centering
  \begin{tabular}{l|cccc|c}
    \toprule
    Methods & L100 & L38 & L43 & L46 & Avg.\\
    \midrule
    ERM$^{\dagger}$ & 59.4 $\pm$ 0.9       & 49.3 $\pm$ 0.6       & \underline{60.1 $\pm$ 1.1}       & 43.2 $\pm$ 0.5       & 53.0 \\
    ERM(reproduced) & 57.6 $\pm$ 3.2 & 40.8 $\pm$ 5.8 & 57.2 $\pm$ 1.8 & 38.9 $\pm$ 3.1 & 48.6 \\
    IRM$^{\dagger}$ & 56.5 $\pm$ 2.5       & 49.8 $\pm$ 1.5       & 57.1 $\pm$ 2.2       & 38.6 $\pm$ 1.0       & 50.5 \\
    GroupDRO$^{\dagger}$ & 60.4 $\pm$ 1.5       & 48.3 $\pm$ 0.4       & 58.6 $\pm$ 0.8       & 42.2 $\pm$ 0.8       & 52.4 \\
    Mixup$^{\dagger}$ & \textbf{67.6 \bm{$\pm$} 1.8}       & \underline{51.0 $\pm$ 1.3}       & 59.0 $\pm$ 0.0       & 40.0 $\pm$ 1.1       & \underline{54.4} \\
    MLDG$^{\dagger}$  & 59.2 $\pm$ 0.1       & 49.0 $\pm$ 0.9       & 58.4 $\pm$ 0.9       & 41.4 $\pm$ 1.0       & 52.0 \\
    CORAL$^{\dagger}$ & 60.4 $\pm$ 0.9       & 47.2 $\pm$ 0.5       & 59.3 $\pm$ 0.4       & \textbf{44.4 \bm{$\pm$} 0.4}       & 52.8 \\
    MMD$^{\dagger}$ & 60.6 $\pm$ 1.1       & 45.9 $\pm$ 0.3       & 57.8 $\pm$ 0.5       & 43.8 $\pm$ 1.2       & 52.0 \\
    DANN$^{\dagger}$ & 55.2 $\pm$ 1.9       & 47.0 $\pm$ 0.7       & 57.2 $\pm$ 0.9       & 42.9 $\pm$ 0.9       & 50.6 \\
    CDANN$^{\dagger}$ & 56.3 $\pm$ 2.0       & 47.1 $\pm$ 0.9       & 57.2 $\pm$ 1.1       & 42.4 $\pm$ 0.8       & 50.8 \\
    MTL$^{\dagger}$ & 58.4 $\pm$ 2.1       & 48.4 $\pm$ 0.8       & 58.9 $\pm$ 0.6       & 43.0 $\pm$ 1.3       & 52.2  \\
    SagNet$^{\dagger}$ & 56.4 $\pm$ 1.9       & 50.5 $\pm$ 2.3       & 59.1 $\pm$ 0.5       & 44.1 $\pm$ 0.6       & 52.5 \\
    ARM$^{\dagger}$ & 60.1 $\pm$ 1.5       & 48.3 $\pm$ 1.6       & 55.3 $\pm$ 0.6       & 40.9 $\pm$ 1.1       & 51.2 \\
    VREx$^{\dagger}$ & 56.8 $\pm$ 1.7       & 46.5 $\pm$ 0.5       & 58.4 $\pm$ 0.3       & 43.8 $\pm$ 0.3       & 51.4 \\
    RSC$^{\dagger}$  & 59.9 $\pm$ 1.4       & 46.7 $\pm$ 0.4       & 57.8 $\pm$ 0.5       & \underline{44.3 $\pm$ 0.6}       & 52.1 \\
    SelfReg$^{\ast}$ & 56.8 $\pm$ 0.9 & 44.7 $\pm$ 0.6 & 59.6 $\pm$ 0.3 & 42.9 $\pm$ 0.8 & 51.0 \\
    SWAD$^{\ast}$ & 55.4 $\pm$ 0.0 & 44.9 $\pm$ 1.1 & 59.7 $\pm$ 0.4 & 39.9 $\pm$ 0.2 & 50.0 \\
    SWAD(reproduced) & 57.4 $\pm$ 1.6 & 46.2 $\pm$ 1.4 & 59.7 $\pm$ 0.6 & 39.1 $\pm$ 1.0 & 50.6 \\
    \midrule
    STEP-S(ours) & \underline{62.4 $\pm$ 1.1} & \textbf{54.4 \bm{$\pm$} 1.0} & \textbf{61.7 \bm{$\pm$} 0.9} & 44.1 $\pm$ 0.0 & \textbf{55.6} \\
    \bottomrule
  \end{tabular}
  \label{tab:terrainc}
\end{table*}

\begin{table*}[htb]
  \caption{Full results of out-of-domain accuracies on DomainNet~\cite{Peng_2019_ICCV_domainnet}.}
  \centering
  \begin{tabular}{l|cccccc|c}
    \toprule
    Methods & clipart & infograph & painting & quickdraw & real & sketch & Avg.\\
    \midrule
    ERM$^{\dagger}$ & 58.6 $\pm$ 0.3       & 19.2 $\pm$ 0.2       & 47.0 $\pm$ 0.3       & 13.2 $\pm$ 0.2       & 59.9 $\pm$ 0.3       & 49.8 $\pm$ 0.4       & 41.3 \\
    ERM(reproduced) & 60.9 $\pm$ 0.4 & 19.5 $\pm$ 0.6 & 48.1 $\pm$ 0.6 & 13.0 $\pm$ 0.2 & 62.4 $\pm$ 0.3 & 49.9 $\pm$ 1.3 & 42.3 \\
    IRM$^{\dagger}$ & 40.4 $\pm$ 6.6       & 12.1 $\pm$ 2.7       & 31.4 $\pm$ 5.7       & 9.8 $\pm$ 1.2        & 37.7 $\pm$ 9.0       & 36.7 $\pm$ 5.3       & 28.0 \\
    GroupDRO$^{\dagger}$ & 47.2 $\pm$ 0.5       & 17.5 $\pm$ 0.4       & 34.2 $\pm$ 0.3       & 9.2 $\pm$ 0.4        & 51.9 $\pm$ 0.5       & 40.1 $\pm$ 0.6       & 33.4 \\
    Mixup$^{\dagger}$ & 55.6 $\pm$ 0.1       & 18.7 $\pm$ 0.4       & 45.1 $\pm$ 0.5       & 12.8 $\pm$ 0.3       & 57.6 $\pm$ 0.5       & 48.2 $\pm$ 0.4       & 39.6 \\
    MLDG$^{\dagger}$  & 59.3 $\pm$ 0.1       & 19.6 $\pm$ 0.2       & 46.8 $\pm$ 0.2       & 13.4 $\pm$ 0.2       & 60.1 $\pm$ 0.4       & 50.4 $\pm$ 0.3       & 41.6 \\
    CORAL$^{\dagger}$ & 59.2 $\pm$ 0.1       & 19.9 $\pm$ 0.2       & 47.4 $\pm$ 0.2       & 14.0 $\pm$ 0.4       & 59.8 $\pm$ 0.2       & 50.4 $\pm$ 0.4       & 41.8 \\
    MMD$^{\dagger}$ & 32.2 $\pm$ 13.3      & 11.2 $\pm$ 4.5       & 26.8 $\pm$ 11.3      & 8.8 $\pm$ 2.2        & 32.7 $\pm$ 13.8      & 29.0 $\pm$ 11.8      & 23.5 \\
    DANN$^{\dagger}$ & 53.1 $\pm$ 0.2       & 18.3 $\pm$ 0.1       & 44.2 $\pm$ 0.7       & 11.9 $\pm$ 0.1       & 55.5 $\pm$ 0.4       & 46.8 $\pm$ 0.6       & 38.3 \\
    CDANN$^{\dagger}$ & 54.6 $\pm$ 0.4       & 17.3 $\pm$ 0.1       & 44.2 $\pm$ 0.7       & 12.8 $\pm$ 0.2       & 56.2 $\pm$ 0.4       & 45.9 $\pm$ 0.5       & 38.5 \\
    MTL$^{\dagger}$ & 58.0 $\pm$ 0.4       & 19.2 $\pm$ 0.2       & 46.2 $\pm$ 0.1       & 12.7 $\pm$ 0.2       & 59.9 $\pm$ 0.1       & 49.0 $\pm$ 0.0       & 40.8  \\
    SagNet$^{\dagger}$ & 57.7 $\pm$ 0.3       & 19.1 $\pm$ 0.1       & 46.3 $\pm$ 0.5       & 13.5 $\pm$ 0.4       & 58.9 $\pm$ 0.4       & 49.5 $\pm$ 0.2       & 40.8 \\
    ARM$^{\dagger}$ & 49.6 $\pm$ 0.4       & 16.5 $\pm$ 0.3       & 41.5 $\pm$ 0.8       & 10.8 $\pm$ 0.1       & 53.5 $\pm$ 0.3       & 43.9 $\pm$ 0.4       & 36.0 \\
    VREx$^{\dagger}$ & 43.3 $\pm$ 4.5       & 14.1 $\pm$ 1.8       & 32.5 $\pm$ 5.0       & 9.8 $\pm$ 1.1        & 43.5 $\pm$ 5.6       & 37.7 $\pm$ 4.5       & 30.1 \\
    RSC$^{\dagger}$  & 55.0 $\pm$ 1.2       & 18.3 $\pm$ 0.5       & 44.4 $\pm$ 0.6       & 12.5 $\pm$ 0.1       & 55.7 $\pm$ 0.7       & 47.8 $\pm$ 0.9       & 38.9 \\
    SelfReg$^{\ast}$ & 62.4 $\pm$ 0.1 & \underline{22.6 $\pm$ 0.1} & 51.8 $\pm$ 0.1 & 14.3 $\pm$ 0.1 & 62.5 $\pm$ 0.2 & 53.8 $\pm$ 0.3 & 44.6 \\
    SWAD$^{\ast}$ & \underline{66.0 $\pm$ 0.1} & 22.4 $\pm$ 0.3 & \underline{53.5 $\pm$ 0.1} & \textbf{16.1 \bm{$\pm$} 0.2} & \underline{65.8 $\pm$ 0.4} & \underline{55.5 $\pm$ 0.3} & \underline{46.5} \\
    SWAD(reproduced) & 65.0 $\pm$ 0.4 & 21.9 $\pm$ 0.1 & 52.7 $\pm$ 0.3 & \underline{15.6 $\pm$ 0.2} & 65.2 $\pm$ 0.4 & 55.2 $\pm$ 0.3 & 45.9 \\
    \midrule
    STEP-S(ours) & \textbf{67.3 \bm{$\pm$} 0.2} & \textbf{22.9 \bm{$\pm$} 0.1} & \textbf{54.4 \bm{$\pm$} 0.1} & 15.1 $\pm$ 0.2 & \textbf{67.1 \bm{$\pm$} 0.1} & \textbf{56.7 \bm{$\pm$} 0.2} & \textbf{47.3} \\
    \bottomrule
  \end{tabular}
  \label{tab:domainnet}
\end{table*}

\section{Complexity Analysis.}
\label{appendix:complexity}

STEP does not require additional GPU memory. The average wall-clock training time for one epoch of LP/FT phases is about 50\%/110\% of ERM, respectively. Since the number of iterations of STEP is no more than ERM, their time complexity is roughly equal.

\end{document}